\documentclass[11pt,english]{article}
\usepackage{lmodern}

\usepackage[T1]{fontenc}
\usepackage[latin9]{inputenc}
\usepackage{geometry}
\usepackage{fancybox}
\usepackage{calc}
\usepackage{units}
\usepackage{dsfont}
\usepackage{amsmath}
\usepackage{amsthm}
\usepackage{amssymb}
\usepackage[ruled,vlined]{algorithm2e}
\usepackage{wrapfig}

\usepackage[Symbolsmallscale]{upgreek}

\makeatletter
\numberwithin{figure}{section}
\numberwithin{equation}{section}
\theoremstyle{plain}

\theoremstyle{plain}

\usepackage{comment,url,graphicx,subcaption,relsize}
\usepackage{amssymb,amsfonts,amsmath,amsthm,amscd,dsfont,mathrsfs,mathtools,microtype,nicefrac,pifont}
\usepackage{float,psfrag,epsfig,color,url,hyperref}
\hypersetup{
  colorlinks,
  linkcolor={red!50!black},
  citecolor={blue!50!black},
  urlcolor={blue!80!black}
}
\usepackage{upgreek}
\usepackage[dvipsnames]{xcolor}
\usepackage{epstopdf,bbm,mathtools,enumitem}
\usepackage[toc,page]{appendix}
\usepackage{dsfont}

\usepackage[mathscr]{euscript}
\usepackage{natbib}
\bibpunct{(}{)}{;}{a}{,}{,}
\usepackage{custom}

\makeatletter
\renewcommand{\paragraph}{
  \@startsection{paragraph}{4}
  {\z@}{1.25ex \@plus 1ex \@minus .2ex}{-1em}
  {\normalfont\normalsize\bfseries}
}
\makeatother

\makeatother

\usepackage{babel}
\providecommand{\lemmaname}{Lemma}
\providecommand{\theoremname}{Theorem}

\begin{document}
\def\balign#1\ealign{\begin{align}#1\end{align}}
\def\baligns#1\ealigns{\begin{align*}#1\end{align*}}
\def\balignat#1\ealign{\begin{alignat}#1\end{alignat}}
\def\balignats#1\ealigns{\begin{alignat*}#1\end{alignat*}}
\def\bitemize#1\eitemize{\begin{itemize}#1\end{itemize}}
\def\benumerate#1\eenumerate{\begin{enumerate}#1\end{enumerate}}

\newenvironment{talign*}
 {\let\displaystyle\textstyle\csname align*\endcsname}
 {\endalign}
\newenvironment{talign}
 {\let\displaystyle\textstyle\csname align\endcsname}
 {\endalign}

\def\balignst#1\ealignst{\begin{talign*}#1\end{talign*}}
\def\balignt#1\ealignt{\begin{talign}#1\end{talign}}

\let\originalleft\left
\let\originalright\right
\renewcommand{\left}{\mathopen{}\mathclose\bgroup\originalleft}
\renewcommand{\right}{\aftergroup\egroup\originalright}

\def\Gronwall{Gr\"onwall\xspace}
\def\Holder{H\"older\xspace}
\def\Ito{It\^o\xspace}
\def\Nystrom{Nystr\"om\xspace}
\def\Schatten{Sch\"atten\xspace}
\def\Matern{Mat\'ern\xspace}

\def\tinycitep*#1{{\tiny\citep*{#1}}}
\def\tinycitealt*#1{{\tiny\citealt*{#1}}}
\def\tinycite*#1{{\tiny\cite*{#1}}}
\def\smallcitep*#1{{\scriptsize\citep*{#1}}}
\def\smallcitealt*#1{{\scriptsize\citealt*{#1}}}
\def\smallcite*#1{{\scriptsize\cite*{#1}}}

\def\blue#1{\textcolor{blue}{{#1}}}
\def\green#1{\textcolor{green}{{#1}}}
\def\orange#1{\textcolor{orange}{{#1}}}
\def\purple#1{\textcolor{purple}{{#1}}}
\def\red#1{\textcolor{red}{{#1}}}
\def\teal#1{\textcolor{teal}{{#1}}}

\def\mbi#1{\boldsymbol{#1}} 
\def\mbf#1{\mathbf{#1}}
\def\mrm#1{\mathrm{#1}}
\def\tbf#1{\textbf{#1}}
\def\tsc#1{\textsc{#1}}

\def\mbiA{\mbi{A}}
\def\mbiB{\mbi{B}}
\def\mbiC{\mbi{C}}
\def\mbiDelta{\mbi{\Delta}}
\def\mbif{\mbi{f}}
\def\mbiF{\mbi{F}}
\def\mbih{\mbi{g}}
\def\mbiG{\mbi{G}}
\def\mbih{\mbi{h}}
\def\mbiH{\mbi{H}}
\def\mbiI{\mbi{I}}
\def\mbim{\mbi{m}}
\def\mbiP{\mbi{P}}
\def\mbiQ{\mbi{Q}}
\def\mbiR{\mbi{R}}
\def\mbiv{\mbi{v}}
\def\mbiV{\mbi{V}}
\def\mbiW{\mbi{W}}
\def\mbiX{\mbi{X}}
\def\mbiY{\mbi{Y}}
\def\mbiZ{\mbi{Z}}

\def\textsum{{\textstyle\sum}} 
\def\textprod{{\textstyle\prod}} 
\def\textbigcap{{\textstyle\bigcap}} 
\def\textbigcup{{\textstyle\bigcup}} 

\def\reals{\mathbb{R}} 
\def\integers{\mathbb{Z}} 
\def\rationals{\mathbb{Q}} 
\def\naturals{\mathbb{N}} 
\def\complex{\mathbb{C}} 

\def\what#1{\widehat{#1}}

\def\twovec#1#2{\left[\begin{array}{c}{#1} \\ {#2}\end{array}\right]}
\def\threevec#1#2#3{\left[\begin{array}{c}{#1} \\ {#2} \\ {#3} \end{array}\right]}
\def\nvec#1#2#3{\left[\begin{array}{c}{#1} \\ {#2} \\ \vdots \\ {#3}\end{array}\right]} 

\def\maxeig#1{\lambda_{\mathrm{max}}\left({#1}\right)}
\def\mineig#1{\lambda_{\mathrm{min}}\left({#1}\right)}

\def\Re{\operatorname{Re}} 
\def\indic#1{\mbb{I}\left[{#1}\right]} 
\def\logarg#1{\log\left({#1}\right)} 
\def\polylog{\operatorname{polylog}}
\def\maxarg#1{\max\left({#1}\right)} 
\def\minarg#1{\min\left({#1}\right)} 
\def\Earg#1{\E\left[{#1}\right]}
\def\Esub#1{\E_{#1}}
\def\Esubarg#1#2{\E_{#1}\left[{#2}\right]}
\def\bigO#1{\mathcal{O}\left(#1\right)} 
\def\littleO#1{o(#1)} 
\def\P{\mbb{P}} 
\def\Parg#1{\P\left({#1}\right)}
\def\Psubarg#1#2{\P_{#1}\left[{#2}\right]}
\def\Trarg#1{\Tr\left[{#1}\right]} 
\def\trarg#1{\tr\left[{#1}\right]} 
\def\Var{\mrm{Var}} 
\def\Vararg#1{\Var\left[{#1}\right]}
\def\Varsubarg#1#2{\Var_{#1}\left[{#2}\right]}
\def\Cov{\mrm{Cov}} 
\def\Covarg#1{\Cov\left[{#1}\right]}
\def\Covsubarg#1#2{\Cov_{#1}\left[{#2}\right]}
\def\Corr{\mrm{Corr}} 
\def\Corrarg#1{\Corr\left[{#1}\right]}
\def\Corrsubarg#1#2{\Corr_{#1}\left[{#2}\right]}
\newcommand{\info}[3][{}]{\mathbb{I}_{#1}\left({#2};{#3}\right)} 
\newcommand{\staticexp}[1]{\operatorname{exp}(#1)} 
\newcommand{\loglihood}[0]{\mathcal{L}} 



\newcommand{\fddino}{\text{FD}_{\text{DINOv2}}}

\newcommand{\scalef}{\lambda}
\newcommand{\intf}{\omega}
\newcommand{\intintf}{\Omega}
\newcommand{\noisef}{\eta}
\newcommand{\prenoisef}{\zeta}

\newcommand{\cskip}{c_\text{skip}}
\newcommand{\sigmadata}{\sigma_\text{data}}

\newcommand{\betad}{\beta_d}
\newcommand{\betamax}{\beta_{\text{max}}}
\newcommand{\betamin}{\beta_{\text{min}}}

\newcommand{\Smin}{S_\text{tmin}}
\newcommand{\Smax}{S_\text{tmax}}
\newcommand{\Schurn}{S_\text{churn}}
\newcommand{\Snoise}{S_\text{noise}}

\newcommand{\indicat}{\mathsf{1}}
\newcommand{\Id}{\mathsf{Id}}

\providecommand{\arccos}{\mathop\mathrm{arccos}}
\providecommand{\dom}{\mathop\mathrm{dom}}
\providecommand{\diag}{\mathop\mathrm{diag}}
\providecommand{\tr}{\mathop\mathrm{tr}}
\providecommand{\card}{\mathop\mathrm{card}}
\providecommand{\sign}{\mathop\mathrm{sign}}
\providecommand{\conv}{\mathop\mathrm{conv}} 
\def\rank#1{\mathrm{rank}({#1})}
\def\supp#1{\mathrm{supp}({#1})}

\providecommand{\minimize}{\mathop\mathrm{minimize}}
\providecommand{\maximize}{\mathop\mathrm{maximize}}
\providecommand{\subjectto}{\mathop\mathrm{subject\;to}}

\def\openright#1#2{\left[{#1}, {#2}\right)}

\ifdefined\nonewproofenvironments\else
\ifdefined\ispres\else
\newtheorem{theorem}{Theorem}
\newtheorem{lemma}[theorem]{Lemma}
\newtheorem{corollary}[theorem]{Corollary}
\newtheorem{definition}[theorem]{Definition}
\newtheorem{fact}[theorem]{Fact}
\renewenvironment{proof}{\noindent\textbf{Proof.}\hspace*{.3em}}{\qed \vspace{.1in}}
\newenvironment{proof-sketch}{\noindent\textbf{Proof Sketch}
  \hspace*{1em}}{\qed\bigskip\\}
\newenvironment{proof-idea}{\noindent\textbf{Proof Idea}
  \hspace*{1em}}{\qed\bigskip\\}
\newenvironment{proof-of-lemma}[1][{}]{\noindent\textbf{Proof of Lemma {#1}}
  \hspace*{1em}}{\qed\\}
  \newenvironment{proof-of-proposition}[1][{}]{\noindent\textbf{Proof of Proposition {#1}}
  \hspace*{1em}}{\qed\\}
\newenvironment{proof-of-theorem}[1][{}]{\noindent\textbf{Proof of Theorem {#1}}
  \hspace*{1em}}{\qed\\}
\newenvironment{proof-attempt}{\noindent\textbf{Proof Attempt}
  \hspace*{1em}}{\qed\bigskip\\}
\newenvironment{proofof}[1]{\noindent\textbf{Proof of {#1}}
  \hspace*{1em}}{\qed\bigskip\\}
 
\newtheorem*{remark*}{Remark}
\newenvironment{remark}{\noindent\textbf{Remark.}
  \hspace*{0em}}{\smallskip}
\newenvironment{remarks}{\noindent\textbf{Remarks}
  \hspace*{1em}}{\smallskip}
\fi
\newtheorem{observation}[theorem]{Observation}
\newtheorem{proposition}[theorem]{Proposition}
\newtheorem{claim}[theorem]{Claim}
\theoremstyle{definition}
\newtheorem{assumption}{Assumption}
\newtheorem{example}[theorem]{Example}
\theoremstyle{remark}
\newtheorem{intuition}[theorem]{Intuition}
\fi
\makeatletter
\@addtoreset{equation}{section}
\makeatother
\def\theequation{\thesection.\arabic{equation}}

\newcommand{\cmark}{\ding{51}}

\newcommand{\xmark}{\ding{55}}

\newcommand{\eq}[1]{\begin{align}#1\end{align}}
\newcommand{\eqn}[1]{\begin{align*}#1\end{align*}}
\renewcommand{\Pr}{\mathbb{P}}
\newcommand{\Ex}[1]{\mathbb{E}\left[#1\right]}
\newcommand{\bep}{\bs \eta_{t}^p}
\newcommand{\bex}{\bs \eta_{t}^x}

\newcommand{\ppu}{{\underline{\psi}'}}
\newcommand{\ptu}{\tilde{\underline\psi}}
\newcommand{\phu}{\hat{\underline\psi}}
\newcommand{\pbu}{\breve{\underline\psi}}


\newcommand{\nb}[1]{\textcolor{blue}{[nb: #1]}}

\global\long\def\bw{\textsf{Ball walk}}%
\global\long\def\dw{\textup{\textsf{Dikin walk}}}%
\global\long\def\sw{\textup{\textsf{Speedy walk}}}%

\global\long\def\E{\mathbb{E}}%

\global\long\def\N{\mathbb{N}}%

\global\long\def\R{\mathbb{R}}%

\global\long\def\Z{\mathbb{Z}}%

\global\long\def\veps{\varepsilon}%

\global\long\def\vol{\textrm{vol}}%

\global\long\def\bs#1{\boldsymbol{#1}}%

\global\long\def\eu#1{\EuScript{#1}}%

\global\long\def\mb#1{\mathbf{#1}}%

\global\long\def\mbb#1{\mathbb{#1}}%

\global\long\def\mc#1{\mathcal{#1}}%

\global\long\def\mf#1{\mathfrak{#1}}%

\global\long\def\ms#1{\mathscr{#1}}%

\global\long\def\mss#1{\mathsf{#1}}%

\global\long\def\msf#1{\mathsf{#1}}%

\global\long\def\on#1{\operatorname{#1}}%

\global\long\def\D{\mathrm{d}}%
\global\long\def\grad{\nabla}%
 
\global\long\def\hess{\nabla^{2}}%
 
\global\long\def\lapl{\triangle}%
 
\global\long\def\deriv#1#2{\frac{d#1}{d#2}}%
 
\global\long\def\pderiv#1#2{\frac{\partial#1}{\partial#2}}%
 
\global\long\def\de{\partial}%
\global\long\def\lagrange{\mathcal{L}}%

\global\long\def\Gsn{\mathcal{N}}%
 
\global\long\def\BeP{\textnormal{BeP}}%
 
\global\long\def\Ber{\textnormal{Ber}}%
 
\global\long\def\Bern{\textnormal{Bern}}%
 
\global\long\def\Bet{\textnormal{Beta}}%
 
\global\long\def\Beta{\textnormal{Beta}}%
 
\global\long\def\Bin{\textnormal{Bin}}%
 
\global\long\def\BP{\textnormal{BP}}%
 
\global\long\def\Dir{\textnormal{Dir}}%
 
\global\long\def\DP{\textnormal{DP}}%
 
\global\long\def\Expo{\textnormal{Expo}}%
 
\global\long\def\Gam{\textnormal{Gamma}}%
 
\global\long\def\GEM{\textnormal{GEM}}%
 
\global\long\def\HypGeo{\textnormal{HypGeo}}%
 
\global\long\def\Mult{\textnormal{Mult}}%
 
\global\long\def\NegMult{\textnormal{NegMult}}%
 
\global\long\def\Poi{\textnormal{Poi}}%
 
\global\long\def\Pois{\textnormal{Pois}}%
 
\global\long\def\Unif{\textnormal{Unif}}%

\global\long\def\Abs#1{\left\lvert #1\right\rvert }%
\global\long\def\Par#1{\left(#1\right)}%
\global\long\def\Brack#1{\left[#1\right]}%
\global\long\def\Brace#1{\left\{ #1\right\} }%

\global\long\def\inner#1{\left\langle #1\right\rangle }%
 
\global\long\def\binner#1#2{\left\langle {#1},{#2}\right\rangle }%

\global\long\def\onenorm#1{\norm{#1}_{1}}%
\global\long\def\twonorm#1{\norm{#1}_{2}}%
\global\long\def\infnorm#1{\norm{#1}_{\infty}}%
\global\long\def\fronorm#1{\norm{#1}_{\text{F}}}%
\global\long\def\nucnorm#1{\norm{#1}_{*}}%
\global\long\def\staticnorm#1{\|{#1}\|}%
\global\long\def\statictwonorm#1{\staticnorm{#1}_{2}}%

\global\long\def\dtv#1#2{d_{\textrm{TV}}\Par{#1,#2}}%

\global\long\def\<{}%

\global\long\def\iff{\Leftrightarrow}%
\global\long\def\chooses#1#2{_{#1}C_{#2}}%
 
\global\long\def\defeq{\triangleq}%
\global\long\def\half{\frac{1}{2}}%
 
\global\long\def\nhalf{\nicefrac{1}{2}}%
 
\global\long\def\textint{{\textstyle \int}}%
\global\long\def\texthalf{{\textstyle \frac{1}{2}}}%
 
\global\long\def\textfrac#1#2{{\textstyle \frac{#1}{#2}}}%


\global\long\def\acts{\circlearrowright}%
 
\global\long\def\bun#1#2{e_{#1} \otimes{} e_{#2}}%
 
\global\long\def\closeopen#1{\lbrack#1 \rparen}%
 
\global\long\def\comp{\mss c}%
 
\global\long\def\deq{\coloneqq}%
 
\global\long\def\e{\mathrm{e}}%
 
\global\long\def\eqas{\overset{\text{a.s.}}{=}}%
 
\global\long\def\eqdef{\mathrel{\overset{\makebox[0pt]{\mbox{\normalfont\tiny def.}}}{=}}}%
 
\global\long\def\gcoeff{\genfrac{\lbrack}{\rbrack}{0pt}{}}%
 
\global\long\def\im{\msf i}%
 
\global\long\def\indep{\mathrel{\text{\scalebox{1.07}{\ensuremath{\perp\mkern-10mu \perp}}}}}%
 
\global\long\def\lr{\leftrightarrow}%
 
\global\long\def\lrarrow{\leftrightarrow}%
 
\global\long\def\mmid{\mathbin{\|}}%
 
\global\long\def\openclose#1{\lparen#1 \rbrack}%
 
\global\long\def\relmid{\mathrel{}\middle|\mathrel{}}%
 
\global\long\def\rest{\upharpoonright}%
 
\global\long\def\simiid{\overset{\text{i.i.d.}}{\sim}}%
 
\global\long\def\toae#1{\xrightarrow[#1]{\text{a.e.}}}%
 
\global\long\def\toas#1{\xrightarrow[#1]{\text{a.s.}}}%
 
\global\long\def\toprob#1{\xrightarrow[#1]{\mathbb{P}}}%
 
\global\long\def\wmaj{\mathbin{\prec_{\msf w}}}%
 
\global\long\def\T{\mathsf{T}}%

\global\long\def\dee{\mathop{\mathrm{d}\!}}%
 
\global\long\def\dt{\,\dee t}%
 
\global\long\def\ds{\,\dee s}%
 
\global\long\def\dx{\,\dee x}%
 
\global\long\def\dy{\,\dee y}%
 
\global\long\def\dz{\,\dee z}%
 
\global\long\def\dv{\,\dee v}%
 
\global\long\def\dw{\,\dee w}%
 
\global\long\def\dr{\,\dee r}%
 
\global\long\def\dB{\,\dee B}%
\global\long\def\dW{\,\dee W}%
\global\long\def\dmu{\,\dee\mu}%
 
\global\long\def\dnu{\,\dee\nu}%
 
\global\long\def\domega{\,\dee\omega}%

\global\long\def\smiddle{\mathrel{}|\mathrel{}}%
\global\long\def\qtext#1{\quad\text{#1}\quad}%
\global\long\def\psdle{\preccurlyeq}%
 
\global\long\def\psdge{\succcurlyeq}%
 
\global\long\def\psdlt{\prec}%
 
\global\long\def\psdgt{\succ}%

\global\long\def\boldone{\mbf{1}}%
\global\long\def\ident{\mbf{I}}%

\global\long\def\eqdist{\stackrel{d}{=}}%
 
\global\long\def\todist{\stackrel{d}{\to}}%
 
\global\long\def\eqd{\stackrel{d}{=}}%
 
\global\long\def\independenT#1#2{\mathrel{\rlap{$#1#2$}\mkern4mu {#1#2}}}%

\global\long\def\ind{\mathds{1}}%

\title{Sublinear iterations can suffice even for DDPMs}

\author{
        Matthew S.\ Zhang \\
        \small UToronto \\
        \texttt{\small matthew.zhang@mail.utoronto.ca}
		\and
        Stephen Huan \\
        \small CMU \\
        \texttt{\small slhuan@andrew.cmu.edu}
		\and
        Jerry Huang \\
        \small CMU \\
        \texttt{\small jerryhua@andrew.cmu.edu}
        \and
		Nicholas M.\ Boffi \\
		\small CMU \\
		\texttt{\small nboffi@andrew.cmu.edu} \and
		Sitan Chen \\
		\small Harvard \\
		\texttt{\small sitan@seas.harvard.edu} \and
		Sinho Chewi \\
		\small Yale \\
		\texttt{\small sinho.chewi@yale.edu}
}  

\maketitle
\pagenumbering{gobble}

\begin{abstract}
SDE-based methods such as denoising diffusion probabilistic models (DDPMs) have shown remarkable success in real-world sample generation tasks. Prior analyses of DDPMs have been focused on the exponential Euler discretization, showing guarantees that generally depend at least linearly on the dimension or initial Fisher information. Inspired by works in log-concave sampling~\citep{shen2019randomized}, we analyze an integrator -- the denoising diffusion randomized midpoint method (DDRaM) -- that leverages an additional randomized midpoint to better approximate the SDE. Using a recently-developed analytic framework called the ``shifted composition rule'', we show that this algorithm enjoys favorable discretization properties under appropriate smoothness assumptions, with sublinear $\widetilde{O}(\sqrt{d})$ score evaluations needed to ensure convergence. This is the first sublinear complexity bound for pure DDPM sampling --- prior works which obtained such bounds worked instead with ODE-based sampling and had to make modifications to the sampler which deviate from how they are used in practice. We also provide experimental validation of the advantages of our method, showing that it performs well in practice with pre-trained image synthesis models.
\end{abstract}

\newpage

\tableofcontents
\thispagestyle{empty}

\newpage

\pagenumbering{arabic}  

\section{Introduction}

With the emergence of diffusion models~\citep{sohl2015deep,song2019generative, ho2020denoising, songscore} as the leading paradigm for generative modeling in image~\citep{rombach_high-resolution_2022}, video~\citep{ho_video_2022,blattmann_align_2023}, and molecular generation~\citep{geffner_-proteina_2025, geffnerproteina}, a flurry of recent work has sought to place these models on rigorous footing using mathematical insights from high-dimensional statistics and numerical analysis. An early finding in this line of work was that, given sufficiently accurate score estimation, diffusion models can sample from essentially any probability distribution in $d$ dimensions in \emph{$O(d)$ iterations}~\citep{Chen+23SGM,lee2023convergence, Ben+24Diffusion, conforti2025kl}.

Subsequently, there has been sustained interest in quantitatively tightening this bound. A number of works~\citep{Chen+23FlowODE,li2024sharp,huang2025convergence,guptafaster,JiaLi25InstanceDiffusion, LiJia25ImprovedDiffusion} have proven that for \emph{ODE-based} diffusion samplers, i.e., DDIMs~\citep{song2021denoising}, the lack of stochasticity enables the design and analysis of algorithms that only require a number of iterations that is \emph{sublinear} in $d$. Other works have tried circumventing $O(d)$ complexity by instead bounding the \emph{parallel} complexity of diffusion-based sampling~\citep{chen2024accelerating,guptafaster,zhouparallel}, or by showing that diffusion models can adapt to the \emph{intrinsic dimension} of the distribution~\citep{li2024adapting, Bof+25LowDim, liang2025low,tang2025adaptivity}, offering speedups orthogonal to the original question of tightening the dimension dependence.

For this guiding question, however, remarkably the best known guarantee for \emph{SDE-based} diffusion samplers, i.e., DDPMs~\citep{ho2020denoising}, has remained $O(d)$. In this work, we ask:

\begin{center}
    \textbf{\emph{Can SDE-based diffusion sampling provably achieve sublinear complexity?}}
\end{center}

In practice, SDE-based sampling confers a number of advantages that make this question particularly salient. In image generation, although DDIMs outperform DDPMs in the few-step regime, the performance for the former quickly saturates while the performance for the latter continues to improve as the number of steps increases; see, e.g., ~\citet[Figure 4]{karras2022elucidating} and \citet{songscore, cao2023exploring, gonzalez, nieblessing, deveney2025closing}. 
This observation has been borne out across a range of model scales: even for large-scale latent diffusions, properly tuned SDE-based samplers often obtain higher performance than their deterministic counterparts~\citep{ma2024sit}.
Stochasticity of the sampling steps also plays a crucial rule in leading protein diffusion models~\citep{abramson2024accurate,geffnerproteina} as a way to heuristically trade off between diversity and designability. In stochastic optimal control-based approaches to steering diffusion models~\citep{domingoadjoint}, during fine-tuning it is necessary to work with an SDE-driven base generative process, and the complexity of sampling enters not just at inference time, but during the training of the control policy. Likewise, when using stochastic optimal control to transport a point mass to some target measure~\citep{havensadjoint}, it is trivially necessary to use stochastic dynamics to generate entropy.

So what would it take to break the $O(d)$ barrier? Intuition from the log-concave sampling literature suggests that doing so requires a more refined discretization scheme. One of the most powerful such schemes emerging from that line of work is the \emph{randomized midpoint method}~\citep{shen2019randomized}, which forms the backbone of state-of-the-art bounds for log-concave sampling~\citep{altschuler2024SC3,altschuler2025SC4}. This method has also been used in several recent works on ODE-based diffusion sampling~\citep{guptafaster,JiaLi25InstanceDiffusion, LiJia25ImprovedDiffusion}. To reap the benefits of randomized discretization however, all of them crucially rely on the deterministic nature of the sampling dynamics, combined with periodic injections of noise that are convenient for establishing provable guarantees but which deviate significantly from how diffusion models are implemented in practice. Indeed, it was explicitly listed as an unresolved challenge in the conclusion of~\citet{JiaLi25InstanceDiffusion} to extend these analyses to pure DDPMs, and as we discuss in \S\ref{sec:overview}, this runs into a surprising range of new obstacles.

\subsection{Contributions}
\label{sec:contribs}

In this work, we overcome these obstacles and answer our guiding question in the affirmative. We craft a new analysis framework for DDPMs that successfully interfaces with the randomized midpoint method, allowing us to break the $O(d)$ barrier for SDE-based diffusion sampling. We first informally state our main guarantee:

\begin{theorem}[Informal, see Theorem~\ref{thm:main-varying}]
    Let $\varepsilon > 0$, and let $\pi$ be a data distribution over $\R^d$ with bounded second moment. Suppose we have estimates $(\mathsf{s}_t)$ for its scores $(\nabla \log \pi_t)$ along the Ornstein--Uhlenbeck process that are $\widetilde{O}(\varepsilon)$-accurate in $L^2(\pi_t)$ and $L_t$-Lipschitz for $L_t \lesssim (1 - e^{-2t})^{-1}$. Then, there is a discretization of DDPM that samples from a distribution $\hat{\pi}$ that is $\varepsilon^2$-close in $\KL$ divergence to a distribution $\pi^{\sf approx}$ that is $\varepsilon$-close in $W_2$ to $\pi$, with no more than $\widetilde{O}(\sqrt{d}/\varepsilon)$ sampling steps.
\end{theorem}

\noindent There are two main innovations over prior work. First, state-of-the-art guarantees for DDPMs~\citep{Ben+24Diffusion,conforti2025kl,dtddpm} required $O(d)$ sampling steps. Second, state-of-the-art guarantees for DD\emph{I}Ms that achieved sublinear complexity had to fundamentally modify the sampling algorithm (see \S\ref{sec:related}), whereas we simply work with the standard DDPM reverse process used in practice, suitably discretized.

Our specific choice of discretization, the randomized midpoint method, has been employed in prior work on DDIM sampling~\citep{shen2019randomized,guptafaster,JiaLi25InstanceDiffusion, LiJia25ImprovedDiffusion}, but we provide the first analysis for DDPMs. Traditionally, the advantages of this choice of discretization are clear at the level of coupling-based arguments that bound the $W_2$ distance between the true process and the sampler, but for general, non-log-concave distributions, such arguments cannot be run for too long without incurring exponential blowups. Existing analyses in the diffusion setting sidestep this by artificially injecting noise into the dynamics, allowing one to ``restart'' the coupling. Unfortunately, without this trick, prior methods for analyzing DDPMs -- which are rooted in $\TV$ / $\KL$-based analysis -- seem to be fundamentally incompatible with randomized midpoint. To overcome this, we build upon the \emph{shifted composition method}~\citep{altschuler2024SC1}, a powerful new technique from the log-concave sampling literature that combines the advantages of coupling-based $W_2$ analysis with those of information-theoretic $\TV$ / $\KL$ analysis. We defer a more comprehensive overview of our techniques to \S\ref{sec:overview}.

\noindent \textbf{On the smoothness assumption.} The main caveat relative to the prior $O(d)$ guarantees for DDPMs is that we make a smoothness assumption. However, this assumption is weaker than what is made in almost all previous papers on DDIMs that achieve sublinear complexity~\citep{Chen+23FlowODE,guptafaster,LiJia25ImprovedDiffusion}. Those works additionally required smoothness of the true scores, and furthermore the assumed bound was independent of noise scale $t$, whereas our bound on $L_t$ becomes increasingly weaker as $t\to 0$. The one exception is the recent result of~\cite{JiaLi25InstanceDiffusion} for DDIMs; see \S\ref{sec:related} for discussion.

In the absence of any smoothness assumptions, it has remained a central open question in this literature how to obtain sublinear complexity bounds with any score-based algorithm, even an ODE-based one. This is well out of scope of this work, the focus of which is instead on bringing our theoretical understanding of DDPMs closer to what is known for DDIMs. 

\subsection{Comparison to prior work}
\label{sec:related}

Below we describe relevant prior work in the theoretical study of diffusion models. 

\paragraph{Discretization analyses for DDPMs.}

Early work on diffusion model theory focused on convergence guarantees for DDPMs~\citep{block2020generative,de2022convergence, lee2022convergence, liu2022let}, which culminated in the finding by~\citet{Chen+23SGM,lee2023convergence} that they can sample from essentially arbitrary distributions in polynomial time given $L^2$-accurate score estimates. This was subsequently refined by~\citet{chen2023improved} and finally by \citet{Ben+24Diffusion,conforti2025kl} to show convergence in $O(d/\varepsilon^2)$ iterations to a distribution that is $\varepsilon^2$-close in KL to a slight noising of the data distribution. By Pinsker's inequality, this implies $\varepsilon$-closeness in TV, which~\citet{dtddpm} later showed could be obtained using only $O(d/\varepsilon)$ iterations. With the exception of this last work, which exploited a subtle recursive bound on the TV error, all prior works giving convergence guarantees for general distributions relied on Girsanov's theorem.

There have also been a number of works on showing that DDPMs can adapt to low-dimensional structure in the data (see, e.g., \cite{huang2024denoising, li2024adapting, potaptchik2024linear, Bof+25LowDim, liang2025low} and the references therein). These results show that $d$ in the above rates can effectively be replaced with some measure of the \emph{intrinsic dimension $k$} of the distribution; while this is technically ``sublinear'' in the dimension if $k = o(d)$, our sublinear complexity holds even if $k = \Theta(d)$. We leave as an interesting open question how to get $o(k)$ rates using DDPMs. Finally, we remark that there have been various works seeking to modify DDPMs to achieve accelerated rates as a function of $\varepsilon$~\citep[see, e.g.,][] {li2024provable,li2024accelerating,wu2024stochastic}.

\paragraph{Discretization analyses for DDIMs.} As mentioned above, all known diffusion-based sampling guarantees achieving sublinear complexity are based on DDIM sampling. \citet{Chen+23FlowODE} obtained the first sublinear complexity bound of $O(L^2\sqrt{d}/\varepsilon)$ for ODE-based samplers under the assumption that the true scores and the score estimates are $L$-Lipschitz. Their algorithm follows the probability flow ODE but injects randomness by running an \emph{underdamped Langevin corrector} at the end of every time window of length $O(1/L)$. We still refer to such samplers as ODE-based as the randomness is far more intermittent than in a DDPM where Gaussian noise would be added after every $1/\mathrm{poly}(d)$-sized step of the sampler. Nevertheless, this sampling algorithm is a significant deviation from how DDIMs work in practice due to the need for underdamped Langevin correction.

Under the same assumptions, \citet{guptafaster} slightly improved the dimension dependence. \citet{LiJia25ImprovedDiffusion} subsequently obtained dimension dependence of $O(Ld^{1/3}/\varepsilon^{2/3})$ by replacing the underdamped Langevin corrector with Gaussian noise, and with the same algorithmic template, recently~\citet{JiaLi25InstanceDiffusion} achieved $\min(d, L^{1/3} d^{2/3}, L d^{1/3})/\varepsilon^{2/3}$. For the $L$-dependent part of their bound, they only require that the true score is locally Lipschitz with Lipschitz constant scaling similarly to our $L_t$. 
The main novelty of our result is that (1) we show the first sublinear bound for \emph{SDEs}, which answers an open question posed by~\citet{JiaLi25InstanceDiffusion} about analyzing randomized midpoint for pure DDPM-based sampling, and (2) our algorithm simply runs the DDPM reverse process, without any corrector steps. For samplers that purely run the probability flow ODE without corrector steps, \cite{li2024sharp,huang2025convergence} were the first to obtain polynomial convergence bounds without dependence on smoothness, though the best known dimension dependence in this setting is linear.

\paragraph{Randomized midpoint method in sampling.}

The randomized midpoint method was first introduced by~\citet{shen2019randomized} in the context of log-concave sampling with Langevin Monte Carlo. A discussion of its use in that literature would take us too far afield, and we defer to the monograph of~\citet{chewi2025log} for details. We mention, however, that besides the shifted composition method that we apply, there is also a direct $\KL$ analysis of midpoint methods using anticipating Girsanov~\citep{Zhang25Anticipative}, which however cannot achieve sharp rates.
There is also a separate approach in~\citet{kandasamy2024poisson}; see~\citet{altschuler2024SC3} for comparisons and discussion.

In the context of diffusion models, the randomized midpoint method has been incorporated into all recent results on ODE-based sampling with sublinear complexity~\citep{guptafaster,JiaLi25InstanceDiffusion, LiJia25ImprovedDiffusion}. On the empirical front, \citet{kandasamy2024poisson,guptafaster} provided experimental evidence for the favorable scaling of randomized midpoint for diffusion-based sampling.

\paragraph{Concurrent work.} Independently of our work, \cite{jiao2025optimalconvergenceanalysisddpm} also obtained an $O(\sqrt{d})$ iteration complexity for DDPMs using very different techniques.

\section{Preliminaries}

\paragraph{Notation.}

We will use $\upgamma$ to denote a standard Gaussian distribution over $\R^d$. The notation $a = O(b)$ or $a \lesssim b$ means that $a \leq c b$ for an absolute constant $c$ (i.e., not depending on the dimension, accuracy, or smoothness parameters), and similarly $a = \Omega(b), a \gtrsim b$ for $a \geq cb$. $a = \Theta(b)$ or $a \asymp b$ implies $a \lesssim b, a \gtrsim b$ simultaneously. Finally, the notation $\widetilde O, \widetilde \Omega, \widetilde \Theta$ means $O, \Omega, \Theta$ respectively up to extra polylogarithmic factors in $b$.

\paragraph{Denoising diffusions.}

We introduce the formalism of denoising diffusion probabilistic models (DDPMs). Let $\pi_0 \in \mc P(\R^d)$ denote the data distribution.
The forward process is defined by evolving $\pi_0$ along the Ornstein--Uhlenbeck (OU) semigroup, which describes the SDE
\begin{align}\label{eq:OU-forward}\tag{OU}
    \D X_t^\rightarrow = - X_t^\rightarrow \, \D t + \sqrt{2} \, \D B_t^\rightarrow\,, \qquad X_0^\rightarrow \sim \pi_0\,, \qquad \pi_t \deq \law(X_t^\rightarrow)\,,
\end{align}
where $(B_t)_{t \geq 0}$ is a standard Brownian motion. As is well-known by now, this equation admits a time-reversal (with respect to an initial measure $\pi_0$ and terminal time $T \in \R_+$) given by
\begin{align}\label{eq:OU-reverse}\tag{rev-OU}
    \D X_t^\leftarrow &= \bigl\{-X_t^\leftarrow + 2\, \nabla \log \frac{\pi_{T-t}}{\upgamma}(X_t^\leftarrow)\bigr\} \, \D t + \sqrt{2} \, \D B_t^\leftarrow\,,
\end{align}
where $(B_t^\leftarrow)_{t \in [0, T]}$ is another standard Brownian motion. If~\eqref{eq:OU-reverse} is initialized with $X_0^\leftarrow \sim \pi_T$, then $\law(X_t^\leftarrow) = \pi_{T-t}$ for all $t\in [0,T]$. As $\lim_{T \to \infty} \pi_T = \upgamma$, we can view~\eqref{eq:OU-forward} as a stochastic flow of $\pi_0$ to a standard Gaussian, and conversely~\eqref{eq:OU-reverse} as a mechanism for obtaining samples from $\pi_0$ when starting from a standard Gaussian measure, assuming access to the score functions $(\nabla \log \pi_t)_{t \in [0, T]}$ or a suitable approximation. As we will generally be referring to~\eqref{eq:OU-reverse} throughout this work, we will omit the $\cdot^\leftarrow$ in the notation with the reverse temporal direction being assumed.

\paragraph{Algorithm.}

Standard means for approximating~\eqref{eq:OU-reverse} assume that the user has access to a process $(\Ms_t)_{t \in [0, T]}$ where $\Ms_t \approx \nabla \log \pi_t$ in a suitably strong sense. Simply substituting the estimator into~\eqref{eq:OU-reverse} does not define a practical algorithm as the resulting SDE remains non-linear and hence does not admit a closed-form solution in general. Instead, one typically opts to discretize it by an appropriate linearization, for instance the exponential integrator given below. This solves the following SDE on $[t_k, t_{k+1})$ for a sequence of interpolant times $0 = t_0 < t_1 < t_2 < \ldots < t_N \leq T$:
\begin{align}\label{eq:exponential-integrator}\tag{EE}
    \D X^{\operatorname{EE}}_t = \{-X_t^{\operatorname{EE}} + 2\,\tsco_{T-t_k}(X_{t_k}^{\operatorname{EE}}) \} \, \D t + \sqrt{2} \, \D B_t\,.
\end{align}
For convenience, we have defined $\tsco_t \deq \Ms_t - \nabla \log \upgamma$.
Conditional on $X^{\operatorname{EE}}_{t_k}$, this SDE is linear, so we can now compute an exact solution explicitly.

However, intuition from the field of log-concave sampling~\citep{shen2019randomized, altschuler2024SC3} suggests that a \emph{randomized midpoint} discretization can significantly outperform the method above. Define a sequence of random variables $\tau_k$ with distribution function $f_k(\tau) = \frac{e^{\tau-h_k}}{1-e^{-h_k}}$ over $[0, h_k]$. 
Then, the algorithm produces a sequence of iterates $X_{t_k}^\alg$ starting at $X_{t_0}^\alg = X_0^\alg \sim \upgamma$,
as follows: at step $k$ for $k \in [N]$, for $t \in [t_{k-1}, t_{k})$, 
\begin{align}\label{eq:randomized-midpoint}\tag{RMD}
\begin{aligned}
    X_{t}^+ &\deq e^{-(t-t_{k-1})} X_{t_{k-1}}^\alg + 2\,(1-e^{-(t-t_{k-1})}) \,\tsco_{T-t_{k-1}}(X_{t_{k-1}}^\alg) + \sqrt{2} \int_{t_{k-1}}^{t} e^{s-t} \, \D B_s\,,  \\
    X_{t_{k}}^\alg &\deq e^{-h_k} X_{t_{k-1}}^\alg + 2\,(1-e^{-h_k}) \,\tsco_{T-t_{k-1} - \tau_k}(X_{t_{k-1} + \tau_k}^+) + \sqrt{2} \int_{t_{k-1}}^{t_{k}} e^{s-t_{k}} \, \D B_s\,,
\end{aligned}
\end{align}
where $h_k \deq t_{k} - t_{k - 1}$ is the step-size in the $k$-th iteration. 
Note that the two random variables 
\begin{align*}
    \xi_k^+ &\deq \sqrt{2} \int_{t_{k-1}}^{t_{k-1}+\tau_k} e^{s-t_{k-1}-\tau_k } \, \D B_s\,, \qquad \xi_k \deq \sqrt{2} \int_{t_{k-1}}^{t_k} e^{s-t_k} \, \D B_s\,,
\end{align*}
have an explicit distribution that can be easily simulated. See the lemma below.
\begin{lemma}
    For each $(\xi_k^+, \xi_k)$ defined above, we have
    \begin{align*}
        \begin{bmatrix}
            \xi_k^+ \\
            \xi_k
        \end{bmatrix} \sim \mc N\biggl(0\,,\; \begin{bmatrix}
            1-e^{-2\tau_k} & e^{\tau_k-h_k}-e^{-(h_k + \tau_k)} \\
            - & 1-e^{-2h_k}
        \end{bmatrix} \otimes I_d \biggr)\,,
    \end{align*}
    where the missing entry is determined by symmetry.
\end{lemma}
The conditional means of~\eqref{eq:randomized-midpoint} have simple closed forms, and so~\eqref{eq:randomized-midpoint} corresponds to an easily computable Gaussian kernel.

\begin{algorithm}[H]
\caption{Randomized midpoint kernel $P^{\alg}_k$ on $[t_{k-1},t_{k}]$}
\label{eq:rmd-alg}
\KwIn{current state $X^{\alg}_{t_k}\in\mathbb R^d$; step $h_k \deq t_{k}-t_{k-1}$; score map $\Ms_{t}(\cdot)$.}

\BlankLine
\textbf{1. Draw the randomized midpoint}.
Sample $U\sim\mathsf{Unif}(0,1)$ and set
\begin{align*}
    \tau_k=h_k + \log\bigl(1+U\,(e^{-h_k}-1)\bigr) \quad\text{i.e., with density}~f(\tau)=\tfrac{e^{\tau-h_k}}{1-e^{-h_k}}~\text{on}~[0, h_k]\,.
\end{align*}

\textbf{2. Midpoint prediction for $X^+_{t_k+\tau_k}$}.
Draw $Z_1\sim\mathcal N(0,I_d)$ and set the OU noise $\xi_k^+ \deq \sqrt{1-e^{-2\tau_k}}\, Z_1$. Then
\begin{align*}
    X^+_{t_{k-1}+\tau_k} = e^{-\tau_k} X_{t_{k-1}}^\alg + 2\,(1-e^{-\tau_k})\, \tsco_{T-t_{k-1}}(X_{t_{k-1}}^\alg) + \xi_k^+\,.
\end{align*}

\textbf{3. Full-step update for $X^{\alg}_{t_{k+1}}$}.
Draw $Z_2\sim\mathcal N(0,I_d)$ independent of $Z_1$ and set
\begin{align*}
    \xi_k = e^{\tau_k-h_k}\xi_k^+  + \sqrt{1-e^{2(\tau_k-h_k)}}\, Z_2\,.
\end{align*}
Compute the score at the randomized time and update
\begin{align*}
    X^{\alg}_{t_{k}} = e^{-h_k} X_{t_{k-1}}^\alg + 2\,(1-e^{-h_k}) \, \tsco_{T-t_{k-1} - \tau_k}(X_{t_{k-1} + \tau_k}^+) + \xi_k\,.
\end{align*}
\end{algorithm}

\section{Results}\label{sec:results}

We first delineate the assumptions underlying our results. We begin with two relatively benign conditions that are standard in the literature.
\begin{assumption}[$L^2$ accurate estimator]\label{as:l2-acc}
    Assume that for all $t \in [0, T]$, we have
    \begin{align*}
        \E_{\pi_t} [\norm{\nabla \log \pi_t - \Ms_t}^2 ] \leq \epscore^2\,.
    \end{align*}
\end{assumption}
\begin{assumption}[Bounded second moment]\label{as:bdd-moment}
    Assume that the initial distribution has bounded second moment
    \begin{align*}
        \E_{\pi_0}[\norm{\cdot}^2] \leq \texttt M_2^2 < \infty\,.
    \end{align*}
\end{assumption}

\begin{assumption}[Time-varying smoothness]\label{as:vary-lip}
    For all $t \in [0, T]$, the estimated score has a Lipschitz constant bounded as follows: for all $x, y \in \R^d$,
    \begin{align*}
        \norm{\tsco_t(x) - \tsco_t(y)} \leq \frac{\tilde\beta_0\, \norm{x-y}}{1-e^{-2t}}\,.
    \end{align*}
\end{assumption}

As discussed in \S\ref{sec:contribs}, these assumptions are a strict subset of those used in almost all existing works on diffusion-based sampling in sublinear complexity~\citep{Chen+23FlowODE,guptafaster,LiJia25ImprovedDiffusion}, with the exception of the recent work of~\cite{JiaLi25InstanceDiffusion} for which a weaker local Lipschitzness condition sufficed in place of Assumption~\ref{as:vary-lip}. In Appendix~\ref{app:examples} we provide examples of distributions for which the true scores are singular at time $0$ (i.e., not Lipschitz uniformly in time), but which are covered by Assumption~\ref{as:vary-lip}.

\begin{theorem}[Main result]\label{thm:main-varying}
    Suppose that Assumptions~\ref{as:l2-acc},~\ref{as:bdd-moment}, and~\ref{as:vary-lip} hold. Then~\eqref{eq:rmd-alg} with a decaying step size schedule can obtain a sample at time $t_N$ from a distribution $\hat\pi$ such that there is another distribution $\pi^{\msf{approx}}$ with
    \begin{align*}
        \KL(\pi^{\msf{approx}} \mmid \hat\pi) \lesssim \widetilde O\bigl((1+\log^2\{(1\vee\tilde\beta_0)\,(d + \mathtt M_2^2)\} )\,\epscore^2\bigr)\,, \qquad W_2^2(\pi^{\msf{approx}}, \pi_0) \lesssim \epscore^2\,,
    \end{align*}
    for $\epscore \in (0, 1]$, $T \asymp \log \frac{d + \mathtt M_2^2}{\epscore^2} \vee 1$ with no more than the following number of steps:
    \begin{align*}
        N = \widetilde{\Theta}\Bigl(\frac{\tilde \beta_0\sqrt{d + \mathtt M_2^2}}{\epscore}\Bigr)\,.
    \end{align*}
\end{theorem}

\begin{remark}
    In our analysis, we consider an algorithmic variant of~\eqref{eq:randomized-midpoint} wherein each $\tau_k$ is not supported on $h_k$, but rather on a truncation $[0, \varrho_k h_k]$ where $1-\varrho_k \ll 1$ is suitably small. This does not appreciably change the algorithm and is only done for technical convenience.
\end{remark}

Although the nature of the guarantee may initially seem opaque (namely, the existence of an ``intermediate'' measure $\pi^{\msf{approx}}$), we note that standard results in the literature only guarantee $\TV$ (or $\KL$) closeness to the early stopped distribution $\pi_\delta$ for some $\delta > 0$.\footnote{Under Assumption~\ref{as:vary-lip}, this is by necessity, since $\pi_0$ could have singular support in which case $\TV$ closeness to $\pi_0$ is not possible.}
The usual justification for this is that $\pi_\delta$ is close to $\pi_0$ in $W_2$ distance, when $\delta$ is small.
This early stopping assumption is so prevalent that it is often made with little fanfare, but we emphasize this point here to argue that our guarantee ($\KL$-close-to-$W_2$-close) is of the same nature.\footnote{Moreover, it is sufficient to metrize weak convergence. In particular, it controls the bounded Lipschitz distance; see, e.g.,~\citet{Chen+23SGM}.}
We remark, however, that $\pi^{\msf{approx}}$ is constructed from our proof technique and does not correspond to an early stopped distribution.

See Appendix~\ref{app:proofs} for more details on the step size schedules and proofs of the theorems.

\section{Technical overview}
\label{sec:overview}

We first discuss the difficulties inherent in analyzing~\eqref{eq:randomized-midpoint}.
The original analysis of~\citet{shen2019randomized}, which inspired almost all subsequent analyses of randomized midpoint, is based on a coupling argument in $W_2$.
However, \textbf{all state-of-the-art analyses for diffusion models work in $\TV$ or $\KL$}.
When we try to apply the former to the latter, we therefore arrive at a fundamental incongruity.
Indeed, $W_2$ analyses of diffusion models often incur exponential accumulation of errors, unless overly restrictive assumptions such as strong log-concavity are imposed on $\pi_0$, e.g.,~\cite{Bru+25LogConcave, GaoNguZhu25Wass, GaoZhu25ODEWass, YuYu25Wass}.
One way in which existing works achieving sublinear complexity have circumvented this is to introduce \emph{corrector steps} which periodically inject randomness into the dynamics to essentially convert $W_2$ bounds into $\KL$ bounds~\citep{Chen+23FlowODE,guptafaster,JiaLi25InstanceDiffusion, LiJia25ImprovedDiffusion}. This is an option that we cannot afford in this work, as our goal is to simply analyze a discretization of the vanilla DDPM reverse process without further algorithmic modifications.

In light of this, how can we analyze randomized midpoint discretization in $\TV$ or $\KL$?
There are two main challenges.
The first is that standard approaches, such as Girsanov's theorem, do not readily apply to~\eqref{eq:randomized-midpoint}, because natural interpolations of~\eqref{eq:randomized-midpoint} are not Markovian: the intermediate point $X^+_{t_{k-1}+\tau_k}$ ``sees into the future'' for times $t \le t_{k-1} +\tau_k$.

The second challenge is that the analysis should be fairly sharp in order to see a tangible benefit from~\eqref{eq:randomized-midpoint}. Indeed, the intuition behind~\eqref{eq:randomized-midpoint} is that the use of a randomized step size to define $X_{t_{k-1}+\tau_k}^+$ effectively ``debiases'' the algorithm.
This is formalized via the notions of weak and strong errors~\citep{MilTre21StochNum}.
Consider a single iteration on $[t_k, t_{k+1})$, with the random variables $X_{t_{k+1}}^\alg$ and $X_{t_{k+1}}$ obtained by solving~\eqref{eq:randomized-midpoint} and~\eqref{eq:OU-reverse} respectively, from the same initial condition $X_{t_k}^\alg = X_{t_k} = x$.
The weak and strong errors are defined as follows:
\begin{align*}
    \text{Weak error:}&& \norm{\E X^\alg_{t_{k+1}} - \E X_{t_{k+1}}} &\leq \mc E_\weak(x)\,, \\
    \text{Strong error:}&& \norm{X^\alg_{t_{k+1}} - X_{t_{k+1}}}_{L^2} &\leq \mc E_\strong(x)\,.
\end{align*}
These two notions loosely capture the squared ``bias'' and ``variance'' of the discretization scheme at a single step.
When the weak error is substantially smaller than the strong error---as is the case for~\eqref{eq:randomized-midpoint}---then one can prove improved discretization bounds, basically because stochastic fluctuations cancel out \`a la the central limit theorem.\footnote{A simple analogy is that the sum of $N$ i.i.d.\ random variables, each with mean $\mu$ and standard deviation $\sigma$, has size roughly $N\mu+N^{1/2}\sigma$; think of $\mu$ as the weak error and $\sigma$ as the strong error.}
Unfortunately, as can be seen from the definitions, the weak and strong errors are most easily controlled via coupling methods, which are most easily handled in $W_2$.

In summary, we require an analysis framework that works in $\TV$ or $\KL$, which is flexible enough to handle discretizations without Markovian interpolations, and which can witness the benefits of smaller weak errors.

\paragraph{Shifted composition.}
In the literature on log-concave sampling, in which the randomized midpoint discretization first arose, obtaining $\KL$ guarantees was also a longstanding challenge until very recently.
A series of papers~\citep{altschuler2024SC1, altschuler2024SC3, altschuler2025SC4} has developed a new framework, known as \emph{shifted composition}, which satisfies our desiderata above.
We therefore aim to adapt it to the setting of diffusion models.

Briefly, the idea behind shifted composition is that in order to control the $\KL$ divergence between two processes (taken to be the algorithm and the ``ideal'' process it approximates), we can introduce a third process---called the auxiliary process---which is initialized at one of the two processes but is \emph{shifted} to hit the second process at a terminal time.
The hitting condition ensures that the $\KL$ divergence between the two original processes at the terminal time is controlled by the $\KL$ between the first process and the auxiliary process.
Due to the definition of the auxiliary process, this latter $\KL$ can be controlled in terms of a distance recursion which incorporates the weak and strong errors.

\paragraph{Adaptation to diffusion models.}
Although shifted composition is well-suited for our needs, we stress that there are additional technical challenges in the diffusion model setting.
Namely, under Assumption~\ref{as:vary-lip}, the Lipschitz constant is changing with time; moreover, Theorem~\ref{thm:main-varying} uses a non-uniform step size schedule. Accommodating these complications requires an extension of the original shifted composition framework; see Appendix~\ref{app:proofs} for details.

\section{Experiments}
In this section, we perform several experiments in image synthesis using pre-trained models from the EDM codebase~\citep{karras2022elucidating} and the EDM2 evaluation~\citep{karras_analyzing_2024} to validate and extend our theoretical predictions.
We first conduct a baseline comparison demonstrating that~\eqref{eq:rmd-alg} outperforms Euler--Maruyama as well as the exponential Euler integrator applied to~\eqref{eq:OU-reverse}, consistent with the prediction of~\ref{thm:main-varying}.
We then highlight some of the design decisions that go into applying~\eqref{eq:rmd-alg} in practice, where state-of-the-art implementations use stochastic processes distinct from the OU process considered in the theoretical portion of our work.
In this setting, we demonstrate how a tailored adaptation of DDRaM can outperform the Heun sampler introduced in~\citep{karras2022elucidating} even for deterministic ODE sampling. Code for all numerical experiments can be found at \url{https://github.com/stephen-huan/edm_rmd}.

\paragraph{Baseline comparison.}
\begin{figure}[t]
    \centering
    \includegraphics[width=0.31\textwidth]{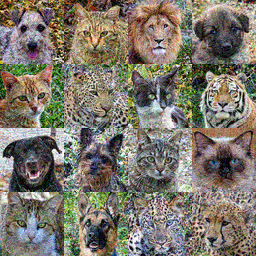}
    \quad%
    \includegraphics[width=0.31\textwidth]{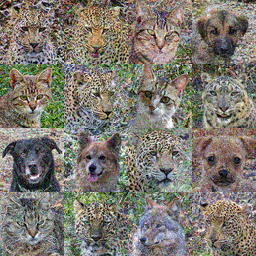}
    \quad%
    \includegraphics[width=0.31\textwidth]{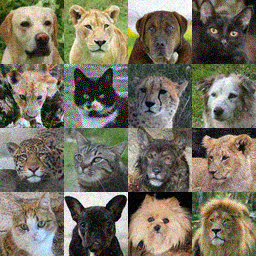}
    \caption{
    \textbf{Qualitative baseline comparison.} 
    Listing from left to right, we show a qualitative comparison between the Euler--Maruyama sampler~\eqref{eq:euler_maru}, the Euler exponential integrator~\eqref{eq:exp_integrator_exact}, and~\eqref{eq:rmd-alg} on the AFHQv2 dataset \citep{choi2020starganv2}. All samplers use 64 score function evaluations (64 Euler integration steps, 32 midpoint steps) and leverage the EDM pre-trained unconditional VP model from~\citep{karras2022elucidating} at \( 64 \times 64 \) resolution over the OU process~\eqref{eq:OU-reverse}.
    Clearly~\eqref{eq:rmd-alg} attains the best visual performance, which we quantify in~\ref{fig:sde_ou}.}
    \label{fig:ou}
\end{figure}
We first compare~\eqref{eq:rmd-alg} to two common baselines---a standard Euler--Maruyama sampler and the exponential Euler sampler.
Specifically, Euler--Maruyama reads
\begin{align}\label{eq:euler_maru}\tag{EMD}
    X_{t_k} = (1 - h_k) X_{t_{k-1}} + 2 h_k\, \bigl(\Ms_{T - t_{k-1}}(X_{t_{k-1}}) + X_{t_{k-1}}\bigr) + \sqrt{2h_k}\,\xi_k\,, \quad \xi_k \sim \gamma \:\text{ i.i.d.}\,,
\end{align}
where the factor $2X_{t_{k-1}}$ originates from the relative score to the standard Gaussian used in~\eqref{eq:OU-reverse}.
Here, we write~\eqref{eq:euler_maru} in terms of the non-relative score because this is what is available as a pre-trained model.
The exponential integrator is given by the analytic solution of~\eqref{eq:exponential-integrator}, which reads
\begin{align}\label{eq:exp_integrator_exact}\tag{EED}
\begin{aligned}
     X_{t_{k}} &= e^{-h_k}X_{t_{k- 1}} + 2\,(1 - e^{-2h_k})\,\bigl(\Ms_{T-t_{k-1}}(X_{t_{k - 1}}) + X_{t_{k - 1}}\bigr) + \sqrt{1 - e^{-2h_k}}\,\xi_k\,,\\
     \xi_k &\sim \gamma\: \text{ i.i.d.}\,.   
\end{aligned}
\end{align}
We note that~\eqref{eq:exp_integrator_exact} can be viewed as a subset of~\eqref{eq:rmd-alg} where we choose $\tau_k = h_k$ deterministically, and where we take $X_{t_{k-1} + \tau_k}^+$ as the next step $X_{t_{k}}$ without an intermediate.
\begin{figure}{r}
\vspace{-7mm}
  \centering
  \includegraphics[width=\linewidth]{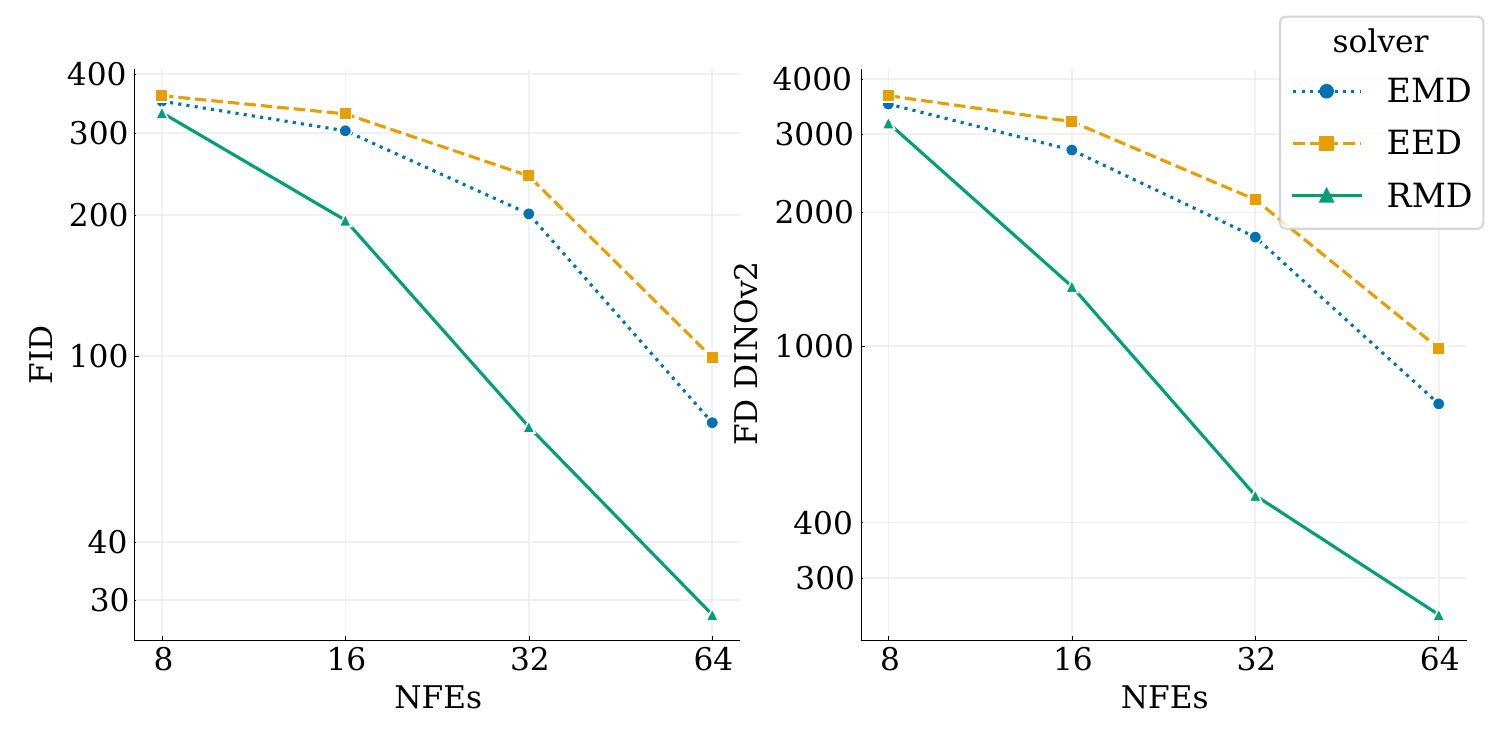}
  \caption{
  \textbf{Quantitative baseline comparison.}
  Image quality measured by FID (top) and \( \fddino \) (bottom) versus number of score function evaluations (NFEs) for the~\eqref{eq:euler_maru},~\eqref{eq:exp_integrator_exact}, and~\eqref{eq:randomized-midpoint} methods run on the OU process.
  Supporting~\ref{fig:ou},~\eqref{eq:randomized-midpoint} obtains the best quantitative results.}
  \label{fig:sde_ou}
\end{figure}
Results for the comparison between~\eqref{eq:rmd-alg},~\eqref{eq:euler_maru} and~\eqref{eq:exp_integrator_exact} are shown in~\ref{fig:ou} and~\ref{fig:sde_ou} on the AFHQv2 dataset~\citep{choi2020starganv2}.
Visually and quantitatively,~\eqref{eq:rmd-alg} performs best of the three.

\paragraph{Beyond the OU process.}

Although theoretical works uniformly analyze the OU process, practitioners often prefer time and space reparametrizations for both training and sampling.
Examples of these include the 
``variance preserving'' (VP) and ``variance exploding'' (VE) SDEs introduced by~\cite{songscore}, as well as the continuous limit of the DDPM schedule \citep{ho2020denoising} suggested by \cite{karras2022elucidating}.
It is \textit{a priori} unclear how to adapt~\eqref{eq:rmd-alg} to these settings, though we may expect to attain similar practical gains to those seen on the OU process given a suitable extension.
In order to extend to these new processes, we use a generalization of the key idea behind~\eqref{eq:rmd-alg} to handle a time-dependent scaling factor \( \scalef(t) \), treating SDEs of the form
\begin{align}\label{eq:general_sde}\tag{SDE}
  \D X_t &= (\scalef(t) X_t + f_t(X_t)) \, \D t + g(t) \, \D B_t\,.
\end{align}
In~\eqref{eq:general_sde}, we have the flexibility to choose \( \scalef(t) \) by appropriate re-definition of $f_t$, which leads to a family of ``randomized midpoint'' methods parameterized by its choice.
The resulting discretization scheme depends on various integrated quantities of \( \scalef(t) \).
For example, the choice \( \scalef(t) = 0 \) generates the randomized midpoint with Euler updates as opposed to the randomized midpoint with \emph{exponential} Euler updates considered in~\eqref{eq:rmd-alg}.
Furthermore, when \( g(t) = 0 \), we notably recover a second-order ODE solver as a special case of the SDE solver.
We provide further details in~\ref{app:rmd_variants}.

Armed with this additional flexibility, we turn to the concrete setting of~\cite{karras2022elucidating}, which considers a reparameterization of \eqref{eq:general_sde} of the form
\begin{align}\label{eq:edm_sde}\tag{EDM}
  \D X_t &= \left[
    \frac{\dot{c}(t)}{c(t)} X_t
    - (c(t)^2 \dot{\sigma}(t) \sigma(t)
    + \beta(t) \sigma(t)^2 c(t)^2)\, \hat{\Ms}_t(X_t)
  \right] \D t + \sqrt{2 \beta(t)} \sigma(t) c(t) \, \D B_t\,,
\end{align}
where we have written \( \hat{\Ms}_t(X_t) \deq \nabla \log \pi(X_t / c(t); \sigma(t)) \) for the (\( c, \sigma \))-parameterized score\footnote{\citet{karras2022elucidating} uses $s(t)$ for the scaling factor rather than $c(t)$. To avoid clash with our notation for the score, we opt to use $c$.}. Note that the VP SDE, VE SDE, and OU process can all be recovered with appropriate choices of \( c \) and \( \sigma \).

We observe that for any choice of \( c \) and \( \sigma \), \eqref{eq:edm_sde} will have: (1) a term with a time-dependent scaling of \( X_t \);
(2) a time-scaling of the score; and (3) a noise term which is a time-dependent scaling of the Wiener process (which is independent of \( X_t \)). 
In our experiments, we mainly consider two natural choices for \( \lambda(t) \) such that the remaining drift of~\eqref{eq:edm_sde} is either a scaling of the score \( \hat{\Ms}_t \) or the relative score \( \tsco_t \), with the time-dependent scaling of \( X_t \) integrated exactly (see~\ref{app:concrete_scale}). 
Our experiments suggest that $\lambda(t)$ should be chosen so the remaining drift is written entirely in terms of the score for the ODE and in terms of the relative score for the SDE.
\begin{figure}[t]
  \centering
  \includegraphics[width=\textwidth]{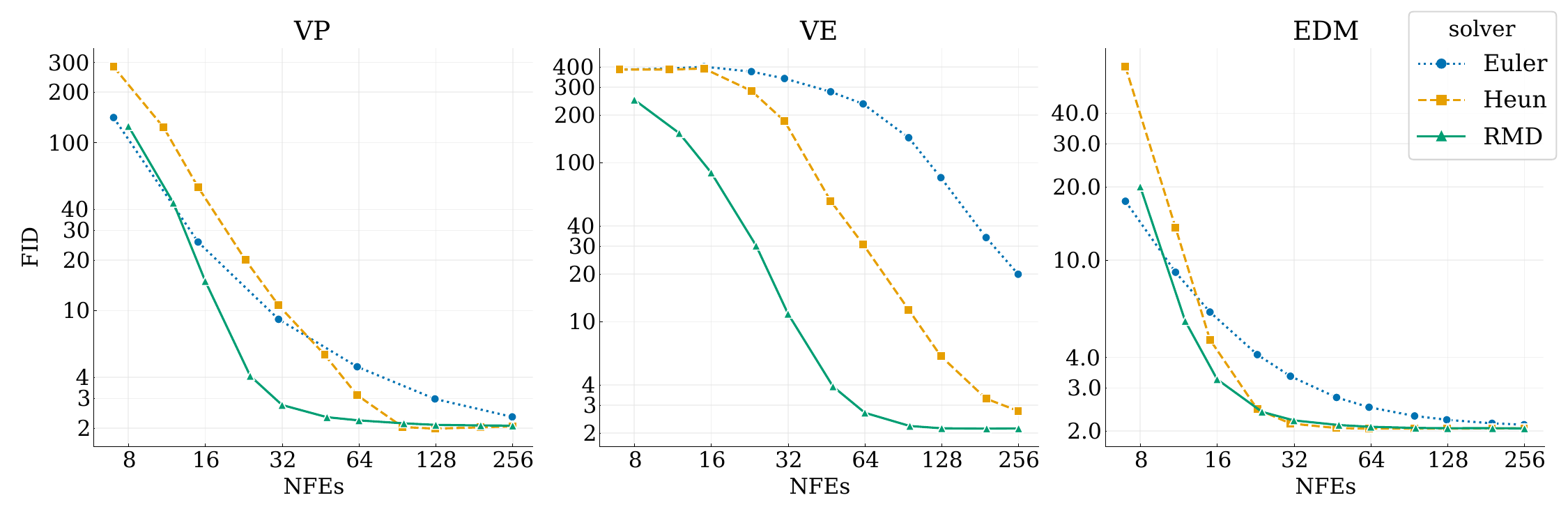}
  \caption{
  \textbf{Quantitative results: Deterministic sampling.}
  Image quality measured by Fr{\'e}chet inception distance (FID$\downarrow$) with number of score function evaluations (NFEs) for the Euler, Heun, and randomized midpoint methods. 
  Columns correspond to the VP, VE, and EDM processes. For \( n \) steps of the solver, Euler takes \( n \) NFEs, Heun takes \( 2 n - 1 \) (since \cite{karras2022elucidating} run Euler on the last step to avoid the singularity at 0), and~\eqref{eq:randomized-midpoint} takes \( 2 n \). As a result, \eqref{eq:randomized-midpoint} has one extra NFE compared to Euler and Heun in these plots. \ref{fig:ode_fd_dinov2} measures using \( \fddino \) on the same images and shows similar results. \ref{fig:ode_fid_combine} shows the NFE curves on a shared \( y \)-axis.}
  \label{fig:ode_fid}
\end{figure}
For our evaluations, we measure the Fr{\'e}chet inception distance (FID) and Fr{\'e}chet distance in the DINOv2 (\( \fddino \)) latent space \citep{Oquab+24DINOv2} as suggested by \cite{stein2023exposing, karras_analyzing_2024} over a batch of $50$k generated samples.
Results are shown in \ref{fig:ode_fid} and \ref{fig:ode_fd_dinov2}, respectively.

We test deterministic sampling ($\beta(t) = 0$) on the AFHQv2 dataset \citep{choi2020starganv2} using the pre-trained VP model from \cite{karras2022elucidating} over the VP, VE, and EDM processes. 
As shown in~\ref{fig:ode_fid} and~\ref{fig:ode_fd_dinov2}, we find that~\eqref{eq:rmd-alg} outperforms both the Euler and Heun samplers at essentially every NFE for all three settings considered in \cite{karras2022elucidating}, highlighting the advantages of DDRaM over widely-adopted diffusion solvers.
We further note that DDRaM empirically seems to be far more robust to the choice of noise scheduler compared to Euler and Heun, where the NFE curves do not vary as much between processes. 
This is clearly seen in~\ref{fig:ode_fid_combine}.

\section{Conclusion}

In this paper, we have shown that stochastic diffusion model samplers can break the $O(d)$ complexity barrier given the right discretization and a natural Lipschitz assumption for the score estimator. Empirically, we find the randomized midpoint performs well in a variety of settings, outperforming both Euler and Heun for both stochastic SDE and deterministic ODE sampling. 
Several interesting lines of exploration remain for future work. 
First, it may be possible to combine our analysis with works that establish discretization guarantees depending only on the \emph{intrinsic} dimension.
Second, it would be quite interesting if Assumptions~\ref{as:vary-lip} on the score estimator could be removed, thereby providing an analysis under the minimal assumptions of~\citet{Ben+24Diffusion, conforti2025kl}.
This may be challenging, as it seems incompatible with our proof technique.
Another possible avenue would be to replace our Lipschitz conditions with a relaxed Lipschitz condition, similarly to~\citet{JiaLi25InstanceDiffusion}, which would imply substantially better guarantees for Gaussian mixtures.

\section*{Acknowledgements}

We thank Jason M.\ Altschuler, Linda Cai, and Shivam Gupta for insightful discussions about randomized midpoint for diffusion models.
MSZ was supported by a NSERC CGS-D award. SC was supported by NSF CAREER award CCF-2441635.
SH was supported by the U.S. Department of Energy, Office of Science, Office of Advanced Scientific Computing Research, Department of Energy Computational Science Graduate Fellowship under Award Number DE-SC0026073.

\paragraph{Disclaimer.} This report was prepared as an account of work sponsored by an agency of the United States Government. Neither the United States Government nor any agency thereof, nor any of their employees, makes any warranty, express or implied, or assumes any legal liability or responsibility for the accuracy, completeness, or usefulness of any information, apparatus, product, or process disclosed, or represents that its use would not infringe privately owned rights. Reference herein to any specific commercial product, process, or service by trade name, trademark, manufacturer, or otherwise does not necessarily constitute or imply its endorsement, recommendation, or favoring by the United States Government or any agency thereof. The views and opinions of authors expressed herein do not necessarily state or reflect those of the United States Government or any agency thereof.

\bibliography{bib}

\appendix

\section{Deferred proofs}\label{app:proofs}

\subsection{Preliminary lemmas}

Before proceeding, the following two generic lemmas will be useful in both regimes. Define
\begin{align*}
    \Mg_t \deq \E_{\pi_{T-t}}\bigl[\bigl\lVert\nabla \log \frac{\pi_{T-t}}{\upgamma}\bigr\rVert^2\bigr]\,.
\end{align*}
We note that $\Mg_t = \FI(\pi_{T-t}\mmid \upgamma)$.

\begin{lemma}[{Magic lemma I, adapted from~\citet[Lemma 5]{conforti2025kl}}]\label{lem:magic-1}
    It holds that, letting $\mathtt M_2^2 \deq \E_{\pi_0}[\norm{\cdot}^2]$ be the initial second moment, 
    \begin{align*}
        \Mg_{t} \lesssim \frac{d}{1-e^{-2(T-t)}} + \mathtt M_2^2 \,.
    \end{align*}
\end{lemma}
The next lemma follows from a computation based on It\^o's lemma.
\begin{lemma}[{Magic lemma II, adapted from~\citet[Proof of Lemma 2]{conforti2025kl}}]\label{lem:magic-2}
    It holds that, for $s < t$, letting $\pi_{T-s,T-t}$ be the joint law of the particle from~\eqref{eq:OU-reverse} at times $\{s, t\}$,
    \begin{align*}
        \E_{(X_s, X_t) \sim \pi_{T-s, T-t}}\bigl[\bigl\lVert\nabla \log \frac{\pi_{T-t}}{\upgamma}(X_t) - \nabla \log \frac{\pi_{T-s}}{\upgamma}(X_s)\bigr\rVert^2\bigr] \lesssim \Mg_{t} - \Mg_{s}\,.
    \end{align*}
\end{lemma}

\subsection{Review of the shifted composition framework}\label{app:review_shift}

The shifted composition framework for proving discretization bounds for Markov processes, developed in the sequence of works~\cite{altschuler2024SC1, altschuler2024SC3, altschuler2025SC4}, allows for the translation of Wasserstein/coupling-based errors to $\KL$ guarantees. They also allow the user to account for the difference between ``weak'' and ``strong'' errors. Suppose first that the following assumptions hold.

\begin{assumption}[Wasserstein regularity results]\label{as:wasserstein-reg}
    Suppose that we have a sequence of kernels $(P_n)_{n \geq 0}$, $(P_n^\alg)_{n \geq 0}$ for which the following properties hold. Namely, for any $n \in \mathbb N$, let $x, y \in \R^d$, and let $X^\alg \sim \delta_x P_n^\alg$, $Y \sim \delta_y P_n$, $Y^\alg \sim \delta_y P_n^\alg$ be coupled. Then, for functions $\Ew, \Es: \R^d \to \R_+$, constants $L, \gamma \geq 0$, assume that the following hold:
    \begin{enumerate}
        \item \textbf{Weak error:} $\norm{\E Y^\alg - \E Y} \leq \Ew(y)$.

        \item \textbf{Strong error:} $\norm{Y^\alg - Y}_{L^2} \leq \Es(y)$ for some coupling of $(Y^\alg, Y)$.

        \item \textbf{Wasserstein Lipschitzness:} $\norm{X^\alg-Y^\alg}_{L^2} \leq L\, \norm{x-y}$ for some coupling of $(X^\alg, Y^\alg)$.

        \item \textbf{Coupling:} $\norm{X^\alg-x-(Y^\alg-y)}_{L^2} \leq \gamma\, \norm{x-y}$ for some coupling of $(X^\alg, Y^\alg)$.
    \end{enumerate}
    Without loss of generality in this work, assume $L \geq 1$.
\end{assumption}

Additionally, some conditions on the $\KL$ divergence are necessary to obtain guarantees.

\begin{assumption}[KL regularity]\label{as:kl-reg}
    With the same notation as Assumption~\ref{as:wasserstein-reg}, assume that the following holds: for a parameter $c \geq 0$ and all $n \in \mathbb N$, $\KL(\delta_x P_n^\alg \mmid \delta_y P_n^\alg) \leq c\, \norm{x-y}^2$.
\end{assumption}

Then, the following guarantee holds.
\begin{lemma}
    Under Assumptions~\ref{as:wasserstein-reg} and~\ref{as:kl-reg}, if $L \geq 1 + \frac{\log N}{N}$, we have for some measure $\pi^{\mathrm{approx}}$ and any initial measure $\pi \in \mc P(\R^d)$ that, defining $\bar P_k^\alg \deq P_1^\alg \dotsm P_k^\alg$ and $\bar P_k \deq P_1 \dotsm P_k$,
    \begin{align*}
        \KL(\pi^{\mathrm{approx}} \mmid \pi \bar P^\alg_N) \lesssim c\,\bigl(\{(L-1)N \vee \log N\}\,\bar{\mc E}_\strong^2 + (N-1)\, (\bar{\mc E}_\weak^2 + \gamma \bar{\mc E}_\strong^2)\bigr) \,.
    \end{align*}
    Furthermore, we have for $N \geq 2L/(L-1)$,
    \begin{align*}
        W_2^2(\pi^{\mathrm{approx}}, \pi \bar P^\alg_N) \lesssim \bar{\mc E}_\strong^2 + \log\bigl(\frac{L}{L-1}\bigr)\, (\bar{\mc E}_\weak^2 + \gamma \bar{\mc E}_\strong^2)\,.
    \end{align*}
    Here, $\bar{\mc E}_\strong^2 = \max_{k \in [N-1]} \E_{\mu \bar P_k}[\mc E_\strong^2]$ and $\bar{\mc E}_\weak^2$ is similarly defined.
\end{lemma}
The proof of this theorem is accomplished by considering, if $(X^\alg_{t_k})_{k\in [N]}, (Y_{t_k})_{k\in [N]}$ are two processes which are started at the same random variable $X^\alg_0 = Y_0$ and which evolve according to the kernels $(P_k^\alg)_{k \in [N]}$ and $(P_k)_{k \in [N]}$ respectively, a third random variable
\begin{align*}
    \tilde Y_0^\aux \deq Y_0\,, \qquad
    \tilde Y_{t_n}^\aux \deq Y_{t_n}^\aux + \upeta_n\,(Y_{t_n} - Y_{t_n}^\aux)\,, \qquad Y_{t_{n+1}}^\aux \sim P^\alg_{n+1}(\tilde Y_{t_n}^\aux, \cdot)\,,
\end{align*}
for an appropriate sequence of shifts $(\upeta_n)_{n \in [N]}$, and then judiciously applying Assumptions~\ref{as:wasserstein-reg} and~\ref{as:kl-reg}. Note that the framework above does not account for the case where the constants vary between the different indices of the kernels $k \in \mathbb N$. This is the cause of substantial difficulties in our analysis, and will be focal point of our technical efforts.

\subsection{Local error analysis}

We start by establishing local error estimates.
When performing our analysis, we actually consider $\tau_k$ having the distribution function with density $f(\tau) \propto e^{\tau-h_k}$ for $\tau \in [0, \varrho_k h_k)$ for technical reasons. In practice, the choice of $\varrho_k \in (0, 1)$ makes little difference in the resulting bounds, and a more streamlined proof would not require such a truncation. We leave the clarification of this detail to future work. We define $(\tilde \beta_s)_{s \in [0, T]}$ to be an upper bound on the Lipschitz constant for $\tsco$, given by
\begin{align*}
    \tilde \beta_t = \frac{\tilde\beta_0}{1-e^{-2(T-t)}}\,, 
\end{align*}
We assume throughout that $\tilde\beta_0\ge 1$.

We first observe that~\eqref{eq:OU-reverse} can be written
\begin{align*}
    X_{t_k}
    &= e^{-h_k} X_{t_{k-1}} + 2 \int_{t_{k-1}}^{t_k} e^{t-t_k} \,\nabla \log \frac{\pi_{T-t}}{\upgamma}(X_t)\,\D t + \sqrt 2 \int_{t_{k-1}}^{t_k} e^{t-t_k}\,\D B_t\,.
\end{align*}

\begin{lemma}[Pointwise local errors]\label{lem:error-est}
    Consider a fixed iteration $k \in [N]$. Under our previous assumptions, we have the following weak error bound, where $X_{t_k}^\alg(x)$ is from~\eqref{eq:randomized-midpoint} with truncation, and $X_{t_k}(x)$ from~\eqref{eq:OU-reverse}, conditional on $X_{t_{k-1}}^\alg = X_{t_{k-1}} = x$ and solving both equations over $t \in [t_{k-1}, t_k)$ for $h_k \ll 1$:
    \begin{align*}
        \norm{\E X_{t_k}^\alg(x) - \E X_{t_k}(x)}^2 &\lesssim h_k \int_{t_{k-1}}^{t_{k}} \Bigl(\MF_t^2(x) + \tilde\beta_{t_k}^2 h_k \int_{t_{k-1}}^t( \MG_{t_{k-1}, s}^2(x) + \MF_{t_{k-1}}^2(x)) \, \D s \Bigr) \, \D t \\
        &\qquad + (1-\varrho_k)^2\,h_k^2\,\sup_{t \in [t_{k-1}, t_{k}]}\E \norm{\tsco_{T-t}(X_t^+(x))}^2 \,,
    \end{align*}
    where $\MG_{s,t}(x)$, $\MF_t(x)$ are defined as
    \begin{align*}
        \MG_{s,t}^2(x) &\deq \E\bigl[\bigl\lVert\nabla \log \frac{\pi_{T-t}}{\upgamma}(X_t(x)) - \nabla \log \frac{\pi_{T-s}}{\upgamma}(X_s(x))\bigr\rVert^2\bigr]\,, \\
        \MF_t^2(x) &\deq \E[\norm{\Ms_{T-t}(X_t(x)) - \nabla \log \pi_{T-t}(X_t(x))}^2]\,,
    \end{align*}
    for $t \in [t_{k-1}, t_{k}]$, starting from $X_{t_{k-1}} = X_{t_{k-1}}^\alg = x$. Also,
    \begin{align*}
        \E[\norm{X^\alg_{t_{k}}(x) - X_{t_{k}}(x)}^2] &\lesssim h_k^2\, \Bigl( \E\MF_{t_{k-1} + \tau_k}^2(x) + \tilde \beta_{t_k}^2 h_k\, \E \int_{t_{k-1}}^{t_{k-1}+\tau_k}( \MG_{t_{k-1}, s}^2(x) + \MF_{t_{k-1}}^2(x)) \, \D s \Bigr) \\
        &\qquad+ h_k \int_{t_{k-1}}^{t_{k}} \E \MG^2_{t_{k-1} +\tau_k, t}(x) \, \D t\,.
    \end{align*}
\end{lemma}
\begin{proof}
    We will suppress the argument in $X_t^\alg(x)$, $X_t(x)$, $X_t^+(x)$, considering always a fixed starting point $x$. Note that for $h_k \ll 1$,
    \begin{align*}
       \Bigl\lvert \frac{e^{t-h_k}}{1-e^{-h_k}} - \frac{e^{t-h_k}} {\int_0^{\varrho_k h_k} e^{s-h_k} \, \D s }\Bigr\rvert \lesssim \frac{1-\varrho_k}{\varrho_k} \cdot \frac{1}{h_k}\,,
    \end{align*}
    for $t \in [0, \varrho_k h_k)$. On the other hand, the maximum of $\frac{e^{t-h_k}}{1-e^{-h_k}}$ on $[\varrho_k h_k, h_k)$ is bounded by at most a constant times $h_k^{-1}$. It follows that, taking $X^\untrc_{t_k}$ from~\eqref{eq:randomized-midpoint} without truncation of the distribution for $\tau_k$, that
    \begin{align*}
        &\norm{\E X^\alg_{t_{k}} - \E X_{t_{k}}}
        \le \norm{\E X^\untrc_{t_{k}} - \E X^\alg_{t_{k}}} + \norm{\E X^\untrc_{t_{k}} - \E X_{t_{k}}} \\
        &\qquad = \Bigl\lVert 2\,(1-e^{-h_k})\,\E\int_0^{h_k} \tsco_{T-t_{k-1}-\tau}(X^+_{t_{k-1}+\tau})\,\Bigl(\frac{e^{\tau-h_k}}{1-e^{-h_k}} - \frac{e^{\tau-h_k} \one_{\tau\le\varrho_k h_k}}{\int_0^{\varrho_k h_k} e^{\tau-h_k}\,\D\tau} \Bigr)\,\D\tau\Bigr\rVert \\
        &\qquad\qquad{} + \Bigl\lVert 2\,\E \int_0^{h_k} \Bigl(\tsco_{T-t_{k-1} +\tau}(X_{t_{k-1}+\tau}^+) - \nabla \log \frac{\pi_{T-t_{k-1}-\tau}}{\upgamma}(X_{t_{k-1}+\tau})\Bigr)\,e^{\tau-h_k}\,\D\tau \Bigr\rVert \\
        &\qquad \lesssim \frac{1-\varrho_k}{\varrho_k} \int_0^{\varrho_k h_k} \E \norm{\tsco_{T-t_{k-1} - t} (X_{t_{k-1} + t}^+)} \, \D t + \int_{\varrho_k h_k}^{h_k} \E \norm{\tsco_{T-t_{k-1} - t}(X_{t_{k-1}+t}^+)}\, \D t\\
        &\qquad\qquad{} + \int_{t_{k-1}}^{t_{k}} \E \norm{\Ms_{T-t}(X_{t}) - \nabla \log \pi_{T-t}(X_{t})} \, \D t
        + \tilde \beta_{t_k} \, \int_{t_{k-1}}^{t_{k}} \E \norm{X_t^+ - X_t} \, \D t\,.
    \end{align*}
    Now, we have
    \begin{align}\label{eq:interp-diff}
    \begin{aligned}
        \E[\norm{X_t^+ - X_t}^2]
        &= \E\Bigl[ \Bigl\lVert\int_{t_{k-1}}^t \Bigl(\tsco_{T-t_{k-1}}(x) - \nabla \log \frac{\pi_{T-s}}{\upgamma}(X_s)\Bigr) \, e^{s-t_k}\, \D s\Bigr\rVert^2\Bigr] \\
        &\lesssim h_k \int_{t_{k-1}}^t\Bigl( \E[\norm{\Ms_{T-t_{k-1}}(x) - \nabla \log \pi_{T-t_{k-1}}(x)}^2] \\
        &\qquad\qquad\qquad{} + \E\Bigl[\Bigl\lVert \nabla \log \frac{\pi_{T-s}}{\upgamma}(X_s) - \nabla \log \frac{\pi_{T-t_{k-1}}}{\upgamma}(x)\Bigr\rVert^2\Bigr]\Bigr) \, \D s\,.
    \end{aligned}
    \end{align}
    On the other hand,
    \begin{align*}
        \norm{X^\alg_{t_{k}} - X_{t_{k}}}^2 &\lesssim \Bigl\lVert\int_{t_{k-1}}^{t_{k}} \bigl(\tsco_{T-t_{k-1} - \tau_k}(X^+_{t_{k-1} + \tau_k}) - \nabla \log \frac{\pi_{T-t}}{\upgamma}(X_t)\bigr) \, e^{t-t_k} \, \D t\Bigr\rVert^2 \\
        &\lesssim h_k \int_{t_{k-1}}^{t_k} \norm{\nabla \log \pi_{T-t}(X_t) - \Ms_{T-t_{k-1} -\tau_k}(X_{t_{k-1} + \tau_k})}^2 \, \D t \\
        &\qquad + \tilde \beta_{t_k}^2 h_k^2\, \norm{X^+_{t_{k-1} + \tau_k} - X_{t_{k-1} + \tau_k}}^2\,.
    \end{align*}
    We then split the first term into
    \begin{align*}
        \E[\norm{\nabla \log \pi_{T-t}(X_t) - \Ms_{T-t_{k-1} -\tau_k}(X_{t_{k-1} + \tau_k})}^2] &\lesssim \MG^2_{t_{k-1}+\tau_k, t}(x) + \MF^2_{t_{k-1} + \tau_k}(x)\,.
    \end{align*}
    For the second term, we can reuse~\eqref{eq:interp-diff}. This gives our desired bound.
\end{proof}

The following lemma follows from applying the Lipschitz property of the estimator, as well as the bound~\eqref{eq:interp-diff} that we previously derived.

\begin{lemma}[Score estimator bounds]\label{lem:score-estimator}
    We have, for $X_t^+(x)$ obtained from~\eqref{eq:randomized-midpoint} conditional on $X_{t_{k-1}}^\alg = x$, for $t \in [t_{k-1}, t_{k}]$,
    \begin{align*}
        \E[\norm{\tsco_{T-t}(X_t^+(x))}^2] &\lesssim \tilde \beta_{t_k}^2 h_k\int_{t_{k-1}}^t \bigl(\MG_{t_{k-1}, s}^2(x) + \MF_{t_{k-1}}^2(x) \bigr) \,\D s \\
        &\qquad + \MG_{t_{k-1}, t}^2(x) +  \MF_t^2(x) + \bigl\lVert\nabla \log \frac{\pi_{T-{t_{k-1}}}}{\upgamma}(x)\bigr\rVert^2\,.
    \end{align*}
\end{lemma}

Recall that the local errors $(\bar{\mc E}^\weak_k)^2$, $(\bar{\mc E}^\strong_k)^2$ are simply the pointwise local errors from Lemma~\ref{lem:error-est}, averaged over $x \sim \pi_{T-t_{k-1}}$.

\begin{lemma}[Local errors]\label{lem:loc_errs}
    For all $k \in \mathbb N$, we have the following errors, taking $1-\varrho_k \asymp h_k^{r}$ for some power $r \geq 2$ at each step (treated as an absolute constant), with $h_k \ll 1/\tilde\beta_{t_k}$ always,
    \begin{enumerate}[label=(\alph*)]
        \item \textbf{Weak error:} 
        \begin{align*}
            (\bar{\mc E}^\weak_k)^2 &\lesssim h_k^2 \epscore^2 + {\tilde \beta_{t_k}^2 h_k^4}\, (\Mg_{t_k}  -\Mg_{t_{k-1}}) + h_k^{2+2r} \Mg_{t_{k}}\,.
        \end{align*}

        \item \textbf{Strong error:}
        \begin{align*}
            (\bar{\mc E}^\strong_k)^2 &\lesssim h_k^2 \epscore^2 + h_k^2\,(\Mg_{t_{k}}  -\Mg_{t_{k-1}}) \,.
        \end{align*}
    \end{enumerate}
    Note that the main difference between the two errors is the additional error term $h_k^2\,(\Mg_{t_k} - \Mg_{t_{k-1}})$ in the strong error.
\end{lemma}
\begin{proof}
    To bound these in expectation, assuming that $X \sim \pi_{T-t_{k-1}}$, we have from Lemma~\ref{lem:magic-2},
    \begin{align*}
        \sup_{t_{k-1} \leq s\leq t \leq t_{k}} \E_{X \sim \pi_{T-t_{k-1}}}[\MG_{s,t}^2(X)] \leq \Mg_{t_{k}} - \Mg_{t_{k-1}}\,.
    \end{align*}
    Here, we note that $t\mapsto \Mg_t$ is monotonically increasing along the Ornstein--Uhlenbeck semigroup.
    
    On the other hand,
    \begin{align*}
         \sup_{t \in [t_{k-1}, t_{k}]} \E_{X\sim\pi_{T-t_{k-1}}}[\MF_t^2(X)] \lesssim \epscore^2\,.
    \end{align*}
    Substituting these into Lemma~\ref{lem:error-est}, and using Lemma~\ref{lem:score-estimator} concludes the proof.
\end{proof}

\subsection{Verifying the assumptions of shifted composition}

Next, we check the hypotheses of the shifted composition local error framework (see Appendix~\ref{app:review_shift}).

\begin{lemma}[Properties of~\eqref{eq:randomized-midpoint}]\label{lem:shifted_comp_assumptions}
    For all $k \in \mathbb N$, the Markov kernels $P_k^\alg$ corresponding to~\eqref{eq:randomized-midpoint} satisfy the following properties, with the same definitions as Lemma~\ref{lem:error-est}.
    Let $Y^\alg$ denote the output of~\eqref{eq:randomized-midpoint} starting from $y$.
    Assume that $h_k \ll 1/\tilde \beta_{t_k}$, and define $\Mp_k \deq{\tilde \beta_{t_k} h_k}$.
    \begin{enumerate}[label=(\alph*)]
        \item \textbf{Wasserstein Lipschitzness:} $
            \norm{X^\alg - Y^\alg}_{L^2} - \norm{x-y} \lesssim \Mp_k\,\norm{x-y}$.

        \item \textbf{Coupling:} $\norm{X^\alg - Y^\alg - (x-y)}_{L^2} \lesssim \Mp_k\, \norm{x-y}$.

        \item \textbf{Regularity:} Let $\varrho_k \in [0, 1)$ be a parameter which is arbitrarily close to $1$. Then, we have
        \begin{align*}
            \KL(\delta_x P_k^\alg \mmid \delta_y P_k^\alg) \lesssim \frac{\norm{x-y}^2}{h_k} \log \frac{1}{1-\varrho_k}\,.
        \end{align*}
    \end{enumerate}
\end{lemma}
\begin{proof}\mbox{}
    \begin{enumerate}[label=(\alph*)]
        \item This follows from (b).

        \item Fixing $\tau_k$ and synchronously coupling the Brownian motions, we have
        \begin{align*}
            \norm{X^\alg - Y^\alg - (x-y)}
            &\le (1-e^{-h_k})\,\norm{x-y} \\
            &\qquad{} + 2\,(1-e^{-h_k})\,\norm{\tsco_{T-t_{k-1}-\tau_k}(X_{t_{k-1}+\tau_k}^+) - \tsco_{T-t_{k-1}-\tau_k}(Y_{t_{k-1}}^+)} \\
            &\le (1-e^{-h_k})\, \norm{x-y} + 2\tilde \beta_{t_k}\, (1-e^{-h_k})\, \norm{X_{t_{k-1} + \tau_k}^+ - Y_{t_{k-1} + \tau_k}^+}\,.
        \end{align*}
        As for the second term, we can bound it again via synchronous coupling:
        \begin{align*}
            \norm{X_{t_{k-1} + \tau_k}^+ - Y_{t_{k-1} + \tau_k}^+}
            &= \norm{e^{-\tau_k}\,(x-y) + 2\,(1-e^{-\tau_k})\,(\tsco_{T-t_{k-1}}(x) - \tsco_{T-t_{k-1}}(y))}\\
            &\lesssim (1+\tilde\beta_{t_k} h_k)\, \norm{x-y}
            \lesssim \norm{x-y}\,.
        \end{align*}

        \item We apply a familiar trick from~\cite{altschuler2024SC3} where we compute the conditional $\KL$ given $\tau_k$, and then integrate. It is for this reason that we need to truncate our random variable $\tau_k$. Condition on $\omega_k \deq \{\tau_k, (B_t)_{t \leq t_{k-1} + \tau_k}\}$. Then, we have
        \begin{align*}
            \delta_x P_{k \mid \omega_k}^\alg = \mc N \bigl(e^{-h_k} x + 2\,(1-e^{-h_k})\, \tsco_{T - t_{k-1} - \tau_k}(X^+_{t_{k-1} + \tau_k}) + \zeta_{k,1},\, (1-e^{-2(h_k-\tau_k)} )\,I_d \bigr)\,,
        \end{align*}
        where
        \begin{align*}
            \zeta_{k, 1} = \sqrt 2\int_{t_{k-1}}^{t_{k-1} + \tau_k} e^{s-t_{k}} \, \D B_s\,.
        \end{align*}
        Using the formula for the $\KL$ divergence between two Gaussians, we find
        \begin{align*}
            &\frac{\norm{e^{-h_k}\, (x-y) + 2\,(1-e^{-h_k})\, (\tsco_{T - t_{k-1} - \tau_k}(X^+_{t_{k-1} + \tau_k}) - \tsco_{T - t_{k-1} - \tau_k}(Y^+_{t_{k-1} + \tau_k}))}^2}{2\,(1-e^{-2(h_k - \tau_k)})} \\
            &\qquad \lesssim \frac{1}{1-e^{-2(h_k - \tau_k)}}\, \norm{x-y}^2 + \frac{\beta_{t_k}^2 h_k^2}{1-e^{-2(h_k-\tau_k)}}\, \norm{X_{t_{k-1}+\tau_k}^+ - Y_{t_{k-1} + \tau_k}^+}^2\\
            &\qquad \lesssim \frac{\norm{x-y}^2}{1-e^{-2(h_k-\tau_k)}}\,.
        \end{align*}
        
        Linearizing the denominator for $h_k \lesssim 1$ and $\tau_k \in [0, \varrho_k h_k]$ for some parameter $\varrho_k$ approaching $1$,
        \begin{align*}
            \KL(\delta_x P_{k \mid \omega_k}^\alg \mmid \delta_y P_{k \mid \omega_k}^\alg) \lesssim \frac{\norm{x-y}^2}{h_k - \tau_k}\,.
        \end{align*}
        Taking expectations and using joint convexity, we find
        \begin{align*}
            \KL(\delta_x P_{k}^\alg \mmid \delta_y P_k^\alg)
            &\lesssim \frac{\norm{x-y}^2}{h_k} \log \frac{1}{1-\varrho_k}\,.
        \end{align*}
    \end{enumerate}
\end{proof}

\subsection{Integral computations}

Now, we need a bespoke version of the original local error recursion from~\citet{altschuler2024SC3} which holds for the time-varying step sizes considered in this work.
We consider the following step size choice, which satisfies $h_k \ll 1/\tilde\beta_{t_k}$.
\begin{align}\label{eq:step-size}
    h_k \deq \frac{C_h \epscore}{\tilde \beta_0 \, \sqrt{(d+ \mathtt M_2^2)\,T}} \min\{1, T-t_k \} \asymp \frac{\epscore}{\tilde \beta_0 \sqrt{(d+ \mathtt M_2^2)\,T} } \cdot (1-e^{-2(T-t_k)}) \,.
\end{align}
Here, $C_h \asymp 1$ is an absolute constant. Let us briefly justify this. When $T-t_k \leq 1/\tilde \beta_0$, then $\frac{1}{1-e^{-2(T-t_k)}} \asymp \frac{1}{T-t_k}$. Otherwise, $\frac{1}{1-e^{-2(T-t_k)}} \asymp 1$.
We also select the shift
\begin{align*}
    \eta_t = \frac{C_\eta \tilde \beta_0}{1-e^{-2 (T-t)}}\,,
\end{align*}
where again $C_\eta \asymp 1$.

The following proof is heavily based on the argument of~\cite{altschuler2024SC3}. Although we briefly describe the high-level idea in the subsequent proof, a detailed discussion of the shifted composition framework is beyond the scope of this paper and we refer to~\citet{altschuler2024SC3}.

\begin{lemma}\label{lem:kl-recursion-time-vary}
    Under Assumptions~\ref{as:l2-acc},~\ref{as:bdd-moment}, and~\ref{as:vary-lip}, with the choice of step-size given in~\eqref{eq:step-size} and for $T \geq 1$ and $t_N \in (T-\frac{1}{6}, T)$, there exists a probability measure $\pi_{t_N}^\aux$ such that
    \begin{align*}
        \KL(\pi_{t_N}^\aux \mmid \pi_{t_N}^\alg) &\lesssim \KL(\pi_T \mmid \upgamma) + \Bigl(T + \frac{1}{T} \log \frac{1}{T-t_N}\Bigr)\,\epscore^2 \log \frac{\tilde \beta_0 \sqrt{(d+ \mathtt M_2^2)\,T}}{\epscore\, (T-t_N)} \,.
    \end{align*}
    Furthermore, if we consider $d_N^2 = \norm{Y^\aux_{t_N} - Y_{t_N}}_{L^2}^2$ where $Y_{t_N}^\aux \sim \pi_{t_N}^\aux$ and $Y_{t_N} \sim \pi_{t_N}$, then
    \begin{align*}
        d_N^2 \lesssim \Bigl( (T-t_N)^2 + \frac{T-t_N}{T}\Bigr)\,\frac{\epscore^2}{\tilde\beta_0^2}\,.
    \end{align*}
\end{lemma}
\begin{proof}
    The idea is to define an auxiliary process $(Y_{t_n}^\aux)_{n \le N}$ with $Y_{t_n}^\aux \sim \pi_{t_N}^\aux$.
    The auxiliary process is defined as follows:
    \begin{align*}
        \tilde Y_0^\aux \sim \pi_T\,, \qquad
        \tilde Y_{t_n}^\aux \deq Y_{t_n}^\aux + \upeta_n\,(Y_{t_n} - Y_{t_n}^\aux)\,, \qquad Y_{t_{n+1}}^\aux \sim P^\alg_{n+1}(\tilde Y_{t_n}^\aux, \cdot)\,.
    \end{align*}
    Here, $\upeta_n \deq \int_{t_{n-1}}^{t_n} \eta_t\,\D t$, and $(Y_t)_{t\in [0,T]}$ denotes~\eqref{eq:OU-reverse}.
    In other words, the auxiliary process follows the~\eqref{eq:randomized-midpoint} algorithm (i.e., using an estimated score and time discretization), but we interleave steps which shift the auxiliary process toward the true reverse process.

    By the KL chain rule,
    \begin{align*}
        \KL(\pi_{t_N}^\aux \mmid \pi_{t_N}^\alg)
        &\le \KL(\pi_T \mmid \upgamma) + \E_{x\sim \pi_T} \KL(\mb P^{\aux}_x \mmid \mb P^{\alg}_x)\,,
    \end{align*}
    where $\mb P^{\aux}_x$, $\mb P^{\alg}_x$ denote path measures started from $x$.

    Define $d_n^2 \deq \E[\norm{Y^\aux_n - Y_{t_n}}^2]$ and note that $d_0 = 0$. We compute the KL divergence between the auxiliary process and the algorithm using the shifted composition technique and Lemma~\ref{lem:shifted_comp_assumptions}; see~\citet[\S 3]{altschuler2024SC3}.
    \begin{align*}
        \E_{x\sim \pi_T} \KL(\mb P^{\aux}_x \mmid \mb P^{\alg}_x)
        \lesssim \sum_{k=1}^N \frac{\upeta_k^2 d_k^2}{h_k} \log \frac{1}{h_k}
        \lesssim \sum_{k=1}^N h_k \eta_{kh}^2 d_k^2 \log \frac{1}{h_k}\,.
    \end{align*}
    The next step is to simplify the computation by approximating the sum by an integral, as was done in~\citet{altschuler2025SC4}.
    In this proof, we reserve the $\texttt{mathtt}$ font for continuous-time interpolations of discrete quantities appearing in this proof.
    Thus, $\mathtt d_t^2$ interpolates $d_n^2$, i.e., $\mathtt d_t^2 \deq d_{t_n}^2$ where $t_n \le t \le t_{n+1}$. Similarly, $\mathtt h_t$ is defined similarly to $h_k$ in~\eqref{eq:step-size}, replacing $t_k$ with $t$.
    Then,
    \begin{align*}
        \E_{x\sim \pi_T} \KL(\mb P^{\aux}_x \mmid \mb P^{\alg}_x)
        \lesssim \int_0^{t_N} \eta_t^2 \mathtt d_t^2 \log \frac{1}{\mathtt h_t} \, \D t\,.
    \end{align*}
    
    We next write down a recursion for $d_n^2$. This is the usual local error recursion, see~\citet[Lemma B.5]{altschuler2024SC3}.
    \begin{align*}
        d_{n}^2 \leq (1+\Mp_{n})^2\, (1-\upeta_n)^2\, d_{n-1}^2 + 2\,(\bar {\mc E}_{n}^\weak + \Mp_n \bar{\mc E}_n^\strong)\, (1-\upeta_n)\, d_{n-1} + (\bar{\mc E}_n^\strong)^2\,.
    \end{align*}
    Here, we invoke Lemma~\ref{lem:shifted_comp_assumptions}, the conclusion of which involves hidden universal constants. By redefining $\Mp_n$ (so that $\Mp_n = O(\tilde\beta_{t_n} h_n)$), we write the above recursion without any universal constants, which simplifies the following computations.
    
    Applying Young's inequality on the middle term, we find that
    \begin{align*}
        d_{n}^2
        &\leq (1+\Mp_{n})\, (1-\upeta_n)\, d_{n-1}^2
        + O\Bigl(\frac{(\bar {\mc E}_{n}^\weak + \Mp_n \bar{\mc E}_n^\strong)^2}{(1+\Mp_{n})\, (1-\upeta_n) - (1+\Mp_{n})^2\, (1-\upeta_n)^2} + (\bar{\mc E}_n^\strong)^2\Bigr)\,.
    \end{align*}
    To simplify the denominator, let us make the ansatz (which we will verify later) that $\Mp_n, \upeta_n \ll 1$ and $(1+\Mp_n)\,(1-\upeta_n) < 1$.
    This then yields the following recursion, noting that $d_0^2 = 0$ is assumed:
    \begin{align*}
        d_n^2 \lesssim \sum_{k=1}^n \Bigl(\prod_{j = k+1}^{n} (1+\Mp_j)\, (1-\upeta_j) \Bigr) \,\Bigl(\frac{(\bar{\mc E}_k^\weak)^2}{\upeta_k - \Mp_k} + (\bar{\mc E}_k^\strong)^2 \Bigr)\,.
    \end{align*}
    In such a case, given our choice of step size and shift, defining
    \begin{align*}
        \mathtt p_t \asymp \frac{\tilde \beta_0 \mathtt h_t}{1-e^{-2(T-t_k)}}\,, \qquad \mathtt h_t \deq \frac{C_h \epscore\, (1-e^{-2(T-t)})}{\tilde \beta_0  \sqrt{(d+ \mathtt M_2^2) T}}\,, 
    \end{align*}
    so that naturally $\Mp_k = \mathtt p_{t_k}$, $h_k = \mathtt h_{t_k}$, we can write
    \begin{align*}
        \upeta_k - \Mp_k 
        \asymp \frac{\tilde \beta_0 \mathtt h_{t_k}}{1-e^{-2(T-t_k)}} \asymp \frac{\epscore}{ \sqrt{(d+ \mathtt M_2^2)T}}\,,
    \end{align*}
    under our choices as well.
    We indeed have $(1+\Mp_n)\,(1-\upeta_n) < 1$ if we choose $C_\eta$ to be a sufficiently large absolute constant, and $\Mp_n, \upeta_n \ll 1$.

    Furthermore, define the following:
    \begin{align*}
        (\mathtt E_t^{\strong})^2 \deq \epscore^2 \, \frac{\epscore\,(1-e^{-2(T-t)})}{\tilde \beta_0  \sqrt{(d+ \mathtt M_2^2) T}} + \frac{\epscore^2\, (1-e^{-2(T-t)})^2}{\tilde \beta^2_0\, (d+ \mathtt M_2^2)T}\, \partial_t \Mg_t\,, \\
        (\mathtt E_t^{\weak})^2 \deq \epscore^2 \, \frac{\epscore\,(1-e^{-2(T-t)})}{\tilde \beta_0  \sqrt{(d+ \mathtt M_2^2) T}} + \frac{\epscore^4\, (1-e^{-2(T-t)})^2}{\tilde \beta^2_0\, (d+ \mathtt M_2^2)^{2} T^2}\, \partial_t \Mg_t\,.
    \end{align*}
    These are obtained by taking the local errors from Lemma~\ref{lem:loc_errs}, dividing by a factor of $h_k$ (which is helpful when converting from the summation to the integral approximation), and taking the continuous-time interpolation.
    Here, the contribution from the $h_k^{2r}$ term can be seen to be negligible, taking $r \geq 4$ sufficiently large and bounding $\Mg$ using Lemma~\ref{lem:magic-1}. Note that the finite difference $\Mg_{t_k} - \Mg_{t_{k-1}}$ converts into a derivative. Finally, we have for absolute constants $\bar c, c$ that
    \begin{align*}
        \prod_{j = k+1}^{n} (1+\Mp_j)\, (1-\upeta_j) &\leq \exp\Bigl(\sum_{j=k+1}^n \bigl( c\tilde \beta_{t_j} h_j -\frac{C_\eta \tilde \beta_0 h_j}{c\,(1-e^{-2(T-t_j)})}\bigr) \Bigr) \\
        &\leq \exp\Bigl(-\int_{t_{k+1}}^{t_n} \frac{\bar c C_\eta \tilde\beta_0}{1-e^{-2(T-t)}} \, \D t \Bigr)\,,
    \end{align*}
    so long as we choose $C_\eta$ to be a sufficiently large absolute constant.
    We then substitute this into
    \begin{align*}
        \mathtt d_t^2 &\lesssim \int_0^t \exp\Bigl(-\int_s^t \frac{\bar c C_\eta \tilde\beta_0}{1-e^{-2(T-r)}}\, \D r \Bigr)\, \Bigl((\mathtt E_t^\strong)^2 + \frac{ \sqrt{(d+\mathtt M_2^2)T}}{\epscore}\, (\mathtt E_t^\weak)^2\Bigr)\, \D s\,.
    \end{align*}
    If we substitute in the definitions of $\mathtt E_t^\weak$, $\mathtt E_t^\strong$,
    \begin{align*}
        \mathtt d_t^2 &\lesssim \int_0^t \exp\Bigl(-\int_s^t \frac{\bar c C_\eta \tilde\beta_0}{1-e^{-2(T-r)}}\, \D r \Bigr) \, \Bigl(\epscore^2 \, \frac{1-e^{-2(T-s)}}{\tilde \beta_0} + \frac{\epscore^2\, (1-e^{-2(T-s)})^2}{\tilde \beta_0^2\, (d+ \mathtt M_2^2) T}\, \partial_s \Mg_s\Bigr) \, \D s \\
        &=\int_0^t \Bigl(\frac{e^{2(T-t)} - 1}{e^{2(T-s)} - 1} \Bigr)^{\frac{\bar c C_\eta \tilde\beta_0}{2}} \,\Bigl(\epscore^2 \, \frac{1-e^{-2(T-s)}}{\tilde \beta_0} + \frac{\epscore^2 \,(1-e^{-2(T-s)})^2}{\tilde \beta_0^2\, (d+ \mathtt M_2^2) T}\, \partial_s \Mg_s\Bigr) \, \D s\,.
    \end{align*}
    Let us now simplify some of these integrals.
    First, for $K \deq \bar c C_\eta \tilde \beta_0 \gg 1$, and using the change of variables $v = e^{-2(T-s)}$, $\D v = 2v\,\D s$,
    \begin{align*}
        \int_0^t \Bigl(\frac{e^{2(T-t)} - 1}{e^{2(T-s)} - 1} \Bigr)^{\frac{\bar c C_\eta \tilde\beta_0}{2}} (1-e^{-2(T-s)}) \, \D s
        &= (e^{2(T-t)}-1)^K \int_0^t (v^{-1}-1)^{-K}\,v\,(v^{-1} -1)\,\D s \\
        &= \frac{(e^{2(T-t)}-1)^K}{2} \int_{e^{-2T}}^{e^{-2(T-t)}} \bigl(\frac{v}{1-v}\bigr)^{K-1} \D v\,.
    \end{align*}
    Next, let $\omega \deq e^{1/K}$.
    \begin{align*}
        &\int_{e^{-2T}}^{e^{-2(T-t)}} \bigl(\frac{v}{1-v}\bigr)^{K-1} \D v
        = \sum_{j\ge 0} \int_{\omega^j \le (1-v)/(1-e^{-2(T-t)}) \le \omega^{j+1}}\bigl(\frac{v}{1-v}\bigr)^{K-1} \D v \\
        &\qquad{}
        \le \frac{1}{(1-e^{-2(T-t)})^{K-1}} \sum_{j\ge 0} \frac{1}{\omega^{(K-1)j}} \int_{\omega^j \le (1-v)/(1-e^{-2(T-t)}) \le \omega^{j+1}} v^{K-1}\, \D v \\
        &\qquad{} \le \frac{1}{(1-e^{-2(T-t)})^{K-1}} \sum_{j\ge 0} \frac{e^{-2(K-1)(T-t)}}{\omega^{(K-1)j}} \,(1-e^{-2(T-t)})\,\omega^j\,(\omega-1) \\
        &\qquad{} \lesssim \frac{e^{-2(K-1)(T-t)}}{K\,(1-e^{-2(T-t)})^{K-1}} \sum_{j\ge 0} \frac{1-e^{-2(T-t)}}{\omega^{(K-2)j}}
        \lesssim \frac{e^{-2(K-1)(T-t)}\,(1-e^{-2(T-t)})}{K\,(1-e^{-2(T-t)})^{K-1}}\,.
    \end{align*}
    On the other hand, a na\"{\i}ve bound is
    \begin{align*}
        \int_{e^{-2T}}^{e^{-2(T-t)}} \bigl(\frac{v}{1-v}\bigr)^{K-1} \D v
        &\le \frac{\int_{e^{-2T}}^{e^{-2(T-t)}} v^{K-1}\,\D v}{(1-e^{-2(T-t)})^{K-1}}
        \le \frac{e^{-2K(T-t)}}{K\,(1-e^{-2(T-t)})^{K-1}}\,.
    \end{align*}
    Using the na\"{\i}ve bound for $T-t \gtrsim 1$, and the refined bound for $T-t \lesssim 1$, we obtain
    \begin{align*}
        \int_0^t \Bigl(\frac{e^{2(T-t)} - 1}{e^{2(T-s)} - 1} \Bigr)^{\frac{\bar c C_\eta \tilde\beta_0}{2}} (1-e^{-2(T-s)}) \, \D s
        &\lesssim \frac{(1-e^{-2(T-t)})^2}{\tilde\beta_0}\,.
    \end{align*}
    On the other hand, integrating by parts, letting
    \begin{align*}
        f(s,t) \deq \Bigl(\frac{e^{2(T-t)} - 1}{e^{2(T-s)} - 1} \Bigr)^{\frac{\bar c C_\eta \tilde\beta_0}{2}}\, (1-e^{-2(T-s)})^2
        = e^{-4(T-s)}\,\frac{(e^{2(T-t)}-1)^K}{(e^{2(T-s)}-1)^{K-2}}
    \end{align*}
    which is increasing in $s$,
    \begin{align*}
        \int_0^{t} \Bigl(\frac{e^{2(T-t)} - 1}{e^{2(T-s)} - 1} \Bigr)^{\frac{\bar c C_\eta \tilde\beta_0}{2}}\, (1-e^{-2(T-s)})^2\, \partial_s \Mg_s \, \D s &= f(t, t)\, \Mg_t - f(0, t)\, \Mg_0 - \int_0^{t} \partial_s f(s, t)\, \Mg_s \, \D s \\
        &\le f(t,t)\,\Mg_t\,.
    \end{align*}
    Together with Lemma~\ref{lem:magic-1}, it yields
    \begin{align*}
        \int_0^{t} \Bigl(\frac{e^{2(T-t)} - 1}{e^{2(T-s)} - 1} \Bigr)^{\frac{\bar c C_\eta \tilde\beta_0}{2}}\, (1-e^{-2(T-s)})^2\, \partial_s \Mg_s \, \D s
        &\lesssim d\,(1-e^{-2(T-t)}) + \mathtt M_2^2\,(1-e^{-2(T-t)})^2\,.
    \end{align*}
    Finally, this all implies that 
    \begin{align*}
        \mathtt d_t^2
        &\lesssim \frac{\epscore^2\,(1-e^{-2(T-t)})^2}{\tilde\beta_0^2} + \frac{\epscore^2}{\tilde\beta_0^2\,(d+\mathtt M_2^2)T}\,[d\,(1-e^{-2(T-t)}) + \mathtt M_2^2\,(1-e^{-2(T-t)})^2]\,.
    \end{align*}
    Now, note that our bound on the $\KL$ divergence is given by
    \begin{align*}
        \E_{x\sim \pi_T} \KL(\mb P^{\aux}_x \mmid \mb P^{\alg}_x)
        &\lesssim \int_0^{t_N} \eta_t^2 \mathtt d_t^2 \log \frac{1}{\mathtt h_t}\,\D t \\
        &\lesssim \bigl(\epscore^2\log \frac{1}{h_N}\bigr) \int_0^{t_N}\Bigl(1 + \frac{d}{(d+\mathtt M_2^2)T\,(1-e^{-2(T-t)})}\Bigr)\,\D t \\
        &\lesssim \bigl(\epscore^2 \log\frac{1}{h_N}\bigr)\,\Bigl(T + \frac{1}{T} \log \frac{e^{2T}-1}{e^{2(T-t_N)}-1} \Bigr)\,. \qedhere
    \end{align*}
\end{proof}
    
\begin{proof}[Proof of Theorem~\ref{thm:main-varying}]
    Lemma~\ref{lem:kl-recursion-time-vary} states that
    \begin{align*}
        \KL(\pi_{t_N}^\aux\mmid \pi_{t_N}^\alg) \lesssim \bigl(\epscore^2 T +  \frac{\epscore^2}{ T} \log \frac{1}{T-t_N}\bigr) \log \frac{\tilde \beta_0 \sqrt{(d + \mathtt M_2^2) T}}{\epscore\, (T-t_N)} + \KL(\pi_{T} \mmid \upgamma)\,.
    \end{align*}
    On the other hand, we have the following:
    \begin{enumerate}[label=(\arabic*)]
        \item For $T-t_N \lesssim 1$, we have
            \begin{align*}
                W_2^2(\pi_0, \pi_{T-t_N}) \lesssim \mathtt M_2^2\, (T-t_N)^2 + d\, (T-t_N)\,.
            \end{align*}

            \item Via Lemma~\ref{lem:kl-recursion-time-vary} again,
            \begin{align*}
                W_2^2(\pi_{t_N}^\aux, \pi_{T-t_N}) \lesssim \epscore^2\, (T-t_N)^2 + \frac{\epscore^2\, (T-t_N)}{T}\,.
            \end{align*}

            \item Lastly, since we use a Gaussian in place of $\pi_T$ as the initial distribution, we need to pay the additional factor
            \begin{align*}
                \KL(\pi_T \mmid \upgamma) \leq e^{-T}
                (d +\mathtt M_2^2) \,,
            \end{align*}
            using~\citet[Lemma 9]{chen2023improved}. So we take $T \asymp \log \frac{d + \mathtt M_2^2}{\epscore^2}$.
    \end{enumerate}
    Thus, we should take $T \asymp \log \frac{d + \mathtt M_2^2}{\epscore^2} \vee 1$, $T-t_N \asymp \frac{\epscore^2}{d} + \frac{\epscore}{\mathtt M_2}$. This all implies that
    \begin{align*}
        W_2^2(\pi_{t_N}^\aux, \pi_0) \lesssim \epscore^2\,, \qquad \KL(\pi_{t_N}^\aux \mmid \pi_{t_N}^{\alg}) = \widetilde O\bigl(\epscore^2\,(1+\log^2\{\tilde \beta_0(d+ \mathtt M_2^2)\})\bigr)\,.
    \end{align*}
    From our choice of step sizes, we note that this takes $N$ steps with
    \begin{align*}
        N \asymp \frac{\tilde \beta_0\sqrt{d+\mathtt M_2^2}\, T^{3/2}}{\epscore} + \frac{\tilde \beta_0\sqrt{(d+\mathtt M_2^2)T}}{\epscore} \log \frac{1}{T-t_N} = \widetilde{\Theta}\Bigl(\frac{\tilde \beta_0\sqrt{d+\mathtt M_2^2}}{ \epscore}\Bigr)\,.
    \end{align*}
\end{proof}

\section{Examples satisfying Assumption~\ref{as:vary-lip}}\label{app:examples}

We provide some examples of distributions where Assumption~\ref{as:vary-lip} holds for the true scores, i.e., for $\Ms_t = \nabla \log \pi_t$.
The following examples all come from the literature on quantitative Lipschitz estimates of Kim--Milman maps (i.e., flow map for the probability flow ODE) which were originally used to establish log-Sobolev inequalities. For completeness, we provide derivations below.
\begin{itemize}[leftmargin=*,itemsep=0pt,topsep=0pt]
    \item \textbf{Log-concave measures.} Let $\pi_0\propto \exp(-V)$ with $\nabla^2 V\succeq 0$.
    Then, Assumption~\ref{as:vary-lip} holds with $\tilde\beta_0 \le 1$.
    \item \textbf{Lipschitz perturbations of strongly log-concave measures.} Let $\pi_0 \propto \exp(-V-W)$, where $V$ is $\alpha$-strongly convex ($\alpha > 0$) and $W$ is $L$-Lipschitz.
    Then, Assumption~\ref{as:vary-lip} holds with $\tilde\beta_0 \le L^2/\alpha \vee 1$.
    \item \textbf{Semi-log-concave over compact sets.} Let $\pi_0\propto \exp(-V)$ over a compact set with diameter at most $R$, and such that $\nabla^2 V \succeq \alpha I_d$ for some $\alpha < 0$.
    Then, Assumption~\ref{as:vary-lip} holds with $\tilde\beta_0 \lesssim 1 \vee \abs\alpha R^2$.
    \item \textbf{Gaussian convolutions of compactly supported measures.} Let $\pi_0 = \nu * \mc N(0, I_d)$, where $\nu$ has compact support, of diameter at most $R$. Then, Assumption~\ref{as:vary-lip} holds with $\tilde\beta_0 \lesssim 1 \vee R^2$.
    \item \textbf{Strongly log-concave outside a ball.} Let $\pi_0\propto \exp(-V)$, where $V$ satisfies
    \begin{align*}
        \inf_{\norm{x-y}=r}\frac{\langle \nabla V(x) - \nabla V(y), x-y\rangle}{\norm{x-y}^2} \ge \begin{cases}
            \alpha - \beta\,, & \norm{x-y}\le R\,, \\
            \alpha\,, & \norm{x-y} > R\,,
        \end{cases}
    \end{align*}
    for some $\alpha,\beta, R > 0$.
    Then, Assumption~\ref{as:vary-lip} holds with some constant $\tilde\beta_0$ depending only on $\alpha$, $\beta$, and $R$.
\end{itemize}

We remark that in all of these examples except the first, the log-Sobolev constant of $\pi_0$ scales exponentially in $\tilde\beta_0$, whereas our convergence bounds only scale polynomially in $\tilde\beta_0$.
This implies that, given access to an accurate score estimator, diffusion models are far superior to standard MCMC methods such as the Langevin diffusion.

We also provide one instance showing the failure of Assumption~\ref{as:vary-lip}.
\begin{itemize}[leftmargin=*,itemsep=0pt,topsep=0pt]
    \item \textbf{Two point masses.} Consider $\pi_0 = \frac{1}{2}\, \delta_{\mathbf e_1} + \frac{1}{2}\, \delta_{-\mathbf e_1}$, where $\mathbf e_1$ is the vector $[1, 0, \dotsc, 0]$. The Hessian is $-\nabla^2 \log \pi_t(\mathbf x) = \frac{1}{1-e^{-2t}}\, I_d -\frac{e^{-2t}}{(1-e^{-2t})^2}\, \mathbf e_1 \mathbf e_1^\top \sech^2(\frac{\mathrm{csch}(t)}{2}\,{\langle{\mathbf e_1, \mathbf x}\rangle})$. Thus, along and near the critical strip $x_1 = 0$, the Hessian experiences blow-up at rate $1/t^2$ as $t \to 0$. This shows that there is no $\tilde \beta_0$ that suffices for all values of $\epscore$.
\end{itemize}
This reasoning can be generalized to other mixtures of point masses.

\subsection{Proofs}

\paragraph{Log-concave measures.}
    Let $\pi_0 \propto \exp(-V)$ where $V: \R^d \to \R$ is strongly log-concave. The conditional distribution of $X_t^\rightarrow$ given $X_0^\rightarrow = x_0$ is $N(e^{-t} x_0, (1-e^{-2t})I_d)$. Using this, standard calculations give that
    \begin{align}\label{eq:hessian}
        \nabla^2 \log \pi_t(x) = -\frac{I_d}{1-e^{-2t}} + \frac{e^{-2t}}{(1-e^{-2t})^2} \operatorname{cov}(X_0^\rightarrow \mid X_t^\rightarrow = x)\,.
    \end{align}
    Now, the reverse conditional measure has the form
    \begin{align*}
        \pi_{0 \mid t}(x \mid y) \propto \exp\Bigl(-\frac{\norm{y-e^{-t} x}^2}{2\,(1-e^{-2t})} - V(x) \Bigr)\,,
    \end{align*}
    so that
        \begin{align}\label{eq:hessian_2}
            -\nabla^2 \log \pi_{0 \mid t}(x \mid y) = \frac{e^{-2t}}{1-e^{-2t}}\,I_d + \nabla^2 V(x) \succeq \frac{e^{-2t}}{1-e^{-2t}}\, I_d\,.
        \end{align}
    The Brascamp--Lieb inequality~\citep{BraLie1976} then allows us to bound the covariance by the inverse of the matrix above. Thus, after some algebra,
        \begin{align*}
            \lambda_{\max}\Bigl(\nabla^2 \log \frac{\pi_t}{\upgamma} \Bigr) = \lambda_{\max}(\nabla^2 \log \pi_t + I_d) \le 1\,.
        \end{align*}
        The minimum eigenvalue can be lower bounded in~\eqref{eq:hessian} by taking the covariance to be zero, which shows that $\tilde\beta_0 = 1$ is sufficient.
        
\paragraph{Lipschitz perturbations of strongly log-concave measures.}
    Next, suppose $\pi_0 \propto \exp(-V - W)$, where $V$ is $\alpha$-strongly convex and $W$ is $L$-Lipschitz. The previous example showed that
    \begin{align*}
        \nabla^2 \log \frac{\pi_t}{\upgamma} = \frac{1}{e^{2t} - 1}\, \Bigl(\frac{\cov_{\nu_{1-e^{-2t}, e^{-t} y}}}{1-e^{-2t}} - I_d\Bigr)\,,
    \end{align*}
    where
    \begin{align*}
            \nu_{\tau, y}(\D x) \propto \exp\Bigl(-\frac{\norm{x-y}^2}{2\tau} + \frac{\norm{x}^2}{2} \Bigr)\, \pi(\D x)\,.
    \end{align*}
    Following the argument of~\citet{BriPed25HeatFlow},
    \begin{align*}
        \norm{\cov_{\nu_{\tau, y}}}_{\operatorname{op}} \leq (\sqrt{\norm{\cov_{\tilde \nu_{\tau, y}}}} + W_2(\nu_{\tau, y}, \tilde \nu_{\tau, y}))^2\,,
    \end{align*}
    where
    \begin{align*}
        \tilde \nu_{\tau, y}(x) \propto \exp\Bigl(-\frac{\norm{x-y}^2}{2\tau} +\frac{\norm{x}^2}{2} - V(x) \Bigr)\,.
    \end{align*}
    Using Brascamp--Lieb, the first term is bounded by         \begin{align*}
            \norm{\cov_{\tilde \nu_{\tau, y}}} \leq \alpha - 1 + \frac{1}{\tau}\,.
    \end{align*}
    On the other hand, by the $\msf T_2$ inequality and LSI,
    \begin{align*}
        W_2^2(\nu_{t,y}, \tilde \nu_{t, y}) \leq C_{\msf{LSI}}^2(\tilde \nu_{t,y})\, \FI(\nu_{t,y} \mmid \tilde \nu_{t, y})\,. 
    \end{align*}
    The Fisher information is the expectation of the squared gradient norm of a $L$-Lipschitz function (namely $W$), whereas we use Bakry--\'Emery to bound the log-Sobolev constant. This all yields the bound on the covariance, for $\tau = 1-e^{-2t}$:
    \begin{align*}
        \norm{\cov_{\nu_{\tau, y}}}
        &\leq \Bigl(\sqrt{\frac{1}{\alpha -1 + \frac{1}{\tau}}} + \frac{L}{\alpha-1 + \frac{1}{\tau}}\Bigr)^2 \\
        &= \Bigl(\sqrt{\frac{1}{\alpha + e^{-2t}/(1-e^{-2t})}} + \frac{L}{\alpha+ e^{-2t}/(1-e^{-2t})}\Bigr)^2 \\
        &\le \Bigl( \sqrt{\frac{1-e^{-2t}}{e^{-2t}}} + \frac{L}{2\sqrt{\alpha e^{-2t}/(1-e^{-2t})}} \Bigr)^2
        \lesssim \bigl(1\vee\frac{L^2}{\alpha}\bigr)\, \frac{1-e^{-2t}}{e^{-2t}}\,.
    \end{align*}
    In particular, this implies the existence of an estimator in Assumption~\ref{as:vary-lip} with $\tilde\beta_0 \lesssim 1 \vee L^2/\alpha$.

\paragraph{Semi-log-concave measures over compact sets.}
Let $\pi_0\propto \exp(-V)$ over a compact set with diameter at most $R$, and such that $\nabla^2 V \succeq \alpha I_d$ for some $\alpha < 0$.
By~\eqref{eq:hessian} and~\eqref{eq:hessian_2}, when $e^{-2t}/(1-e^{-2t}) \ge -2\alpha$, then $\lambda_{\max}(\nabla^2 \log(\pi_t/\upgamma)) \lesssim 1$.
On the other hand, when $e^{-2t}/(1-e^{-2t}) \le -2\alpha$, then
\begin{align*}
    \lambda_{\max}\bigl(\nabla^2 \log \frac{\pi_t}{\upgamma}\bigr) \le \frac{e^{-2t}}{(1-e^{-2t})^2}\,R^2
    \le \frac{-2\alpha R^2}{1-e^{-2t}}\,.
\end{align*}
Putting together the two cases, $\tilde \beta_0 \lesssim 1 \vee \abs\alpha R^2$.
This example and the next are taken from~\citet{MikShe23HeatFlow, MikShe24BrownianTransport}.

\paragraph{Gaussian convolutions of compactly supported measures.}
Let $\pi_0 = \nu * \mc N(0, I_d)$, where $\nu$ has compact support, of diameter at most $R$.
A similar computation to the above examples readily yields
\begin{align*}
    \nabla^2 \log \frac{\pi_t}{\upgamma}
    &\preceq R^2 e^{-2t}\,I_d\,.
\end{align*}
Therefore, we can take $\tilde\beta_0 \lesssim 1 \vee R^2$.

\paragraph{Strongly log-concave outside a ball.}
This example is taken from~\citet{ConLacPal25Entropic}. The constant was not explicitly computed therein in terms of $\alpha$, $\beta$, and $R$.

\section{Experimental details}\label{app:experiments}

\subsection{Adapting the OU process to the EDM framework}\label{app:ou_edm}

Clearly~\eqref{eq:OU-reverse} fits the general SDE~\eqref{eq:general_sde} by taking \( \scalef(t) = -1 \), \( f_t(X_t) = 2\, \nabla \log (\pi_{T - t}/\upgamma) (X_t) \), and \( g(t) = \sqrt{2} \).
We wish to write \eqref{eq:OU-reverse} in terms of~\eqref{eq:edm_sde}. The EDM forward process is defined as \( X_t = c(t)\, X_0 + c(t)\, \sigma(t)\, z \) while \eqref{eq:OU-forward} admits the closed-form solution \( X_t = e^{-t} X_0 + B_{1 - e^{-2 t}} \) where \( B_{\cdot} \) denotes the Wiener process. By comparison, we read \( c(t) = e^{-t} \), \( \sigma(t) = \sqrt{e^{2t} - 1} \). Alternatively, we realize that the OU process is a special case of the VP SDE \citep{songscore} when \( \betamin = \betamax = 2 \), (\( \betad \deq \betamax - \betamin = 0 \)) and read from Table 1 of \cite{karras2022elucidating}. Matching \( \sqrt{2 \beta(t)} \,\sigma(t)\, c(t) \) to \( \sqrt{2} \), we find that \( \beta(t) = (\sigma(t)\, c(t))^{-2} \). Using the relationship between the forward and reverse SDE~\citep[eq.\ (6)]{karras2022elucidating}, we have recovered~\eqref{eq:OU-forward} and~\eqref{eq:OU-reverse}.

It is helpful to remember that the score \( \hat {\Ms}_t(X_t) \) is internally implemented with \emph{denoising score matching} \citep{hyvarinen05a, vincent2011connectiona} and admits the formula
\begin{align}\label{eq:denoising}\tag{EDM-score}
  \hat{\Ms}_t(x) = \frac{D(x / c(t); \sigma(t)) - x / c(t)}{c(t)\, \sigma(t)^2}\,,
\end{align}
where \( D(\cdot; \sigma) \) is a neural network denoiser trained to predict the unnoised \( x \) given \( x + \sigma z \), \( z \sim \gamma \). Writing~\eqref{eq:edm_sde} in terms of the score instead of the denoiser allows for a cleaner implementation which is closer to the SDE, especially for implementing our suggestions around the time scaling \( \scalef(t) \).

\subsection{Variants of the randomized midpoint}\label{app:rmd_variants}

Our starting point is the semi-linear SDE~\eqref{eq:general_sde}
\begin{align}
  \D X_t &= (\scalef(t) X_t + f_t(X_t)) \, \D t + g(t) \, \D B_t\,.
\end{align} From the intuition that a linear SDE of the form \( \D X_t = \scalef(t) X_t \, \D t + g(t) \, \D B_t \) admits a closed-form solution, we use the ODE integrating factor \( \intf(t) \deq \exp(-\int_{t_0}^t \scalef) \) as an ansatz. By \Ito's rule,
\begin{align*}
  \D (\intf(t) X_t) &= (\D \intf(t))\, X_t + \intf(t) \, \D X_t \\
  &= -\scalef(t)\, \intf(t) X_t \, \D t + \intf(t)\, \left[ 
    (\scalef(t) X_t + f_t(X_t)) \, \D t + g(t) \, \D B_t
  \right] \\
  &= \intf(t)\, f_t(X_t) \, \D t + \intf(t)\, g(t) \, \D B_t\,,
\end{align*}
where we have successfully removed the linear term.
Integrating both sides from some starting time \( t_0 \) to \( t_0 + h \), we have the integral representation
\begin{gather}
  \nonumber
  \intf(t_0 + h) X_{t_0 + h} - \intf(t_0) X_{t_0}
  = \int_{t_0}^{t_0 + h} {\intf(t)} \, f_t(X_t) \, \D t
    + \int_{t_0}^{t_0 + h} {\intf(t)}\, g(X_t) \, \D B_t\,, \\
  \label{eq:sde_int}\tag{INT}
  X_{t_0 + h}
  = \frac{\intf(t_0)}{\intf(t_0 + h)} X_{t_0}
    + \int_{t_0}^{t_0 + h} \frac{\intf(t)}{\intf(t_0 + h)}\, f_t(X_t) \, \D t
    + \int_{t_0}^{t_0 + h} \frac{\intf(t)}{\intf(t_0 + h)}\, g(X_t) \, \D B_t\,.
\end{gather}
In order to approximate~\eqref{eq:sde_int} we perform a two-step discretization scheme. First, we draw a random time $\tau$ from the density proportional to \( t \mapsto \indicat_{[t_0, t_0 + h]}(t)\, \intf(t) \) which serves as our midpoint. Defining \( \intintf(t) \deq \int_{t_0}^t \intf \), explicitly
\begin{align*}
  \tau \sim p(\tau) &= \begin{cases}
      \frac{\intf(t)}{\intintf(t_0 + h) - \intintf(t_0)}\,, & t_0 \leq \tau \leq t_0 + h\,; \\
      0\,, & \text{otherwise}\,.
  \end{cases}
\end{align*}
We can use the plug-in estimator \( (\intintf(t_0 + h) - \intintf(t_0)) f_{t_0 + \tau}(X_{t_0 + \tau}) / \intf(t_0 + h) \) to obtain an unbiased estimate to the integral in the fashion of Monte Carlo quadrature. Unfortunately we do not know the value of \( X_{t_0 + \tau} \), necessitating a second approximation. We use an Euler scheme, assuming that the function is constant on the interval and taking the left endpoint (which we do know). Thus, we have
\begin{align*}
  X^+_{t_0 + \tau} &= \frac{\intf(t_0)}{\intf(t_0 + \tau)}\, X_{t_0} 
    + \frac{\intintf(t_0 + \tau)}{\intf(t_0 + \tau)}\, f_{t_0}(X_{t_0})
    + \int_{t_0}^{t_0 + \tau} \frac{\intf(t)}{\intf(t_0 + \tau)}\, g(X_t) \, \D B_t\,, \\
  X_{t_0 + h} &= \frac{\intf(t_0)}{\intf(t_0 + h)}\, X_{t_0} 
    + \frac{\intintf(t_0 + h)}{\intf(t_0 + h)}\, f_{t_0 + \tau}(X_{t_0 + \tau})
    + \int_{t_0}^{t_0 + h} \frac{\intf(t)}{\intf(t_0 + h)}\, g(X_t) \, \D B_t\,.
\end{align*}
It remains to treat the noise terms. Define the stochastic process
\begin{align*}
  Y_{t} &\deq \int_{t_0}^{t} \frac{\intf(t')}{\intf(t_0 + h)}\, g(t') \, \D B_{t'}\,.
\end{align*}
Clearly \( \E[Y_t] = 0 \) and
\begin{align*}
  \Var[Y_t] = \E[Y_t^2]
  = \int_{t_0}^{t} \frac{\intf(t')^2}{\intf(t_0 + h)^2}\, g(t')^2 \, \D t'\,,
\end{align*}
by an application of \Ito's rule. We will compute \( (\xi^+, \xi) \sim (Y_{t_0 + \tau}, Y_{t_0 + h}) \) by conditional simulation. Defining \( \noisef(t) \deq \int_{t_0}^t (\intf g)^2 \); if \( z^+ \sim \gamma \) then \( \xi^+ = [\sqrt{\noisef(t_0 + \tau) - \noisef(t_0)} / \intf(t_0 + \tau)]\, z^+ \) has the right marginal distribution. Next, we perform the domain decomposition 
\begin{align*}
  Y_{t_0 + h} &= \int_{t_0}^{t_0 + \tau} \frac{\intf(t)}{\intf(t_0 + h)}\, g(t) \, \D B_t
    + \int_{t_0 + \tau}^{t_0 + h} \frac{\intf(t)}{\intf(t_0 + h)}\, g(t) \, \D B_t \\
  &= \frac{\intf(t_0 + \tau)}{\intf(t_0 + h)} \int_{t_0}^{t_0 + \tau} \frac{\intf(t)}{\intf(t_0 + \tau)}\, g(t) \, \D B_t
    + \int_{t_0 + \tau}^{t_0 + h} \frac{\intf(t)}{\intf(t_0 + h)}\, g(t) \, \D B_t \\
  &= \frac{\intf(t_0 + \tau)}{\intf(t_0 + h)}\, Y_{t_0 + \tau}
    + \int_{t_0 + \tau}^{t_0 + h}\, \frac{\intf(t)}{\intf(t_0 + h)}\, g(t) \, \D B_t\,,
\end{align*}
and make use of the fact that the latter term is independent of \( Y_{t_0 + \tau} \) and normally distributed with mean 0 and variance \( (\noisef(t_0 + h) - \noisef(t_0 + \tau)) / \intf(t_0 + h)^2 \). Thus, we can compute for \( z \sim \gamma \) independent of \( z^+ \),
\( \xi = [\intf(t_0 + \tau) / \intf(t_0 + h)]\, \xi^+ + [\sqrt{(\noisef(t_0 + h) - \noisef(t_0+\tau))} / \intf(t_0 + h)]\,z \).

Putting everything together, we have the following generalization of~\eqref{eq:rmd-alg}.

\begin{algorithm}[H]
\caption{Generalized randomized midpoint kernel on $[t_0, t_1]$}
\label{eq:rmd-alg-general}
\KwIn{current state $X_{t_0}\in\mathbb R^d$; step $h \deq t_1 - t_0$; drift $f_t(\cdot)$; noise term \( g(\cdot) \).}
\KwIn{scaling factor \( \scalef(t) \); integrating factor \( \intf(t) \deq \exp(-\int_{t_0}^t \scalef) \); normalizing factor \( \intintf(t) \deq \int_{t_0}^t \intf \); inverse \( \intintf^{-1}(\cdot) \); noise factor \( \noisef(t) \deq \int_{t_0}^t (\intf g)^2 \).}

\BlankLine
\textbf{1. Draw the randomized midpoint}.
Sample $U\sim\mathsf{Unif}(0,1)$ and set
\begin{align*}
    \tau = \intintf^{-1}((1 - U)\, \intintf(t_0) + U\, \intintf(t_1))  \quad\text{i.e., with density}~p(\tau) \propto \indicat_{[t_0, t_1]}(t)\, \intf(t).
\end{align*}

\textbf{2. Midpoint prediction for $X^+_{t_0+\tau}$}.
Draw $Z_1\sim\mathcal N(0,I_d)$ and set the OU noise $\xi^+ \deq [\sqrt{\abs{\noisef(t_0 + \tau)}} / \intf(t_0 + \tau)] \, Z_1$. Then
\begin{align*}
    X^+_{t_0+\tau_k} = \frac{\intf(t_0)}{\intf(t_0 + \tau)}\, X_{t_0} 
      + \frac{\intintf(t_0 + \tau)}{\intf(t_0 + \tau)}\, f_{t_0}(X_{t_0}) 
      + \xi^+\,.
\end{align*}

\textbf{3. Full-step update for $X_{t_1}$}.
Draw $Z_2\sim\mathcal N(0,I_d)$ independent of $Z_1$ and set
\begin{align*}
    \xi = \frac{\intf(t_0 + \tau)}{\intf(t_1)}\, \xi^+ + \frac{\sqrt{\abs{\noisef(t_1) - \noisef(t_0+\tau)}}}{\intf(t_1)} \, Z_2\,.
\end{align*}
Compute the score at the randomized time and update
\begin{align*}
    X_{t_1} = \frac{\intf(t_0)}{\intf(t_1)}\, X_{t_0} 
      + \frac{\intintf(t_1)}{\intf(t_1)}\, f_{t_0 + \tau}(X_{t_0 + \tau}^+)
      + \xi\,.
\end{align*}
\end{algorithm}

Note that we take the absolute value in the computation of \( (\xi^+, \xi) \) so that~\eqref{eq:rmd-alg-general} is valid also in reverse time (i.e., when 
\( t_1 < t_0 \) and \( h < 0 \)).

For example, one concrete instantiation as mentioned in the main text is given by 
\begin{align}\label{eq:rm_euler}\tag{RME}
\begin{aligned}
  X_{t_{k - 1} + \tau_k}^+ &= X_{t_{k - 1}} + \tau_k f_{t_{k - 1}}(X_{t_{k - 1}}) + \operatorname{noise}\,, \\
  X_{t_{k}} &= X_{t_{k - 1}} + h_k f_{t_{k - 1} + \tau_k}(X_{t_{k - 1} + \tau_k}) + \operatorname{noise}\,,
\end{aligned}
\end{align}
which corresponds to randomized midpoint without exponential Euler when $\lambda(t) = 0$.

\subsubsection{Implementation details}

The quantity \( \intf(t) \) is free up to multiplicative factors and \( \intintf(t), \noisef(t) \) are free up to constants, assuming they agree with each other. It sometimes convenient to arbitrarily base the integrals at \( t_0 \), i.e.\ to compute \( \intintf(t) = \int_{t_0}^t \intf(t) \, \D t \), resulting in definite integrals for the differences in integrated quantities in~\eqref{eq:rmd-alg-general}. When it is not possible to analytically integrate \( \intf, \intintf \), or \( \noisef \) or to invert \( \intintf \), numerical quadrature and root finding can be used instead. We use \texttt{scipy.integrate.quad} and \texttt{scipy.optimize.root\_scalar} respectively for these tasks, from the SciPy library \citep{2020SciPy-NMeth}. For quadrature it can help to signal discontinuities like \( \Smin \) and \( \Smax \) with the \texttt{points} argument. For root finding we use the \texttt{"brentq"} method with interval \( [t_0, t_1] \). Although in principle we could use a higher-order method like the \texttt{"halley"} method since \( \intintf \) is twice differentiable with derivatives \( \intintf' = \intf \), \( \intintf''(t) = -\scalef(t)\, \intf(t) \), in our settings we find both quadrature and root finding to converge to near machine precision (\( \approx 10^{-11} \)--\( 10^{-15} \)) in a handful of iterations (\( < 10 \)).

\subsubsection{Concrete choices of scaling factor}\label{app:concrete_scale}

Following~\eqref{eq:edm_sde}, we see that in order for the drift to be a time-scaling of the score, it suffices to take \( \scalef(t) = \dot{c}(t) / c(t) \). For the drift to be a time-scaling of the relative score, we take
\begin{align*}
  \scalef(t) = \frac{\dot{c}(t)}{c(t)}
    + \frac{c(t)^2 \dot{\sigma}(t) \sigma(t)}{\sigma_T^2}
    + \frac{c(t)^2 \beta(t) \sigma(t)^2}{\sigma_T^2}\,,
\end{align*}
where \( \sigma_T^2 \) is the variance of the forward process at time \( T \), \( \pi_T \). We also consider a ``network-adapted'' strategy (as opposed to the aforementioned ``SDE-adapted'') strategy by expanding the score in terms of the denoiser~\eqref{eq:denoising} and collecting linear terms, resulting in \( \scalef(t) = \dot{c}(t) / c(t) + \dot{\sigma}(t) / \sigma(t) \). We can also account for the skip connection in the denoiser itself, resulting in the choice of
\begin{align*}
  \scalef(t) = \frac{\dot{c}(t)}{c(t)}
    + (1 - \cskip(t))\, \frac{\dot{\sigma}(t)}{\sigma(t)}\,,
\end{align*}
where \( \cskip(t) \) is the skip connection in the denoiser \( D(\cdot; \sigma) \) (see Table 1 of \cite{karras2022elucidating}). In particular, we consider \( \cskip(t) = \sigmadata^2 / (\sigma(t)^2 + \sigmadata^2) \) for \( \sigmadata = 0.5 \).

In our experiments we use the relative score for the OU and VP processes, the non-relative score for the EDM process, and the skip connection for the VE process.

\clearpage

\subsection{Additional figures}\label{app:add_figures}
\begin{figure}[ht]
  \centering
  \includegraphics[width=\textwidth]{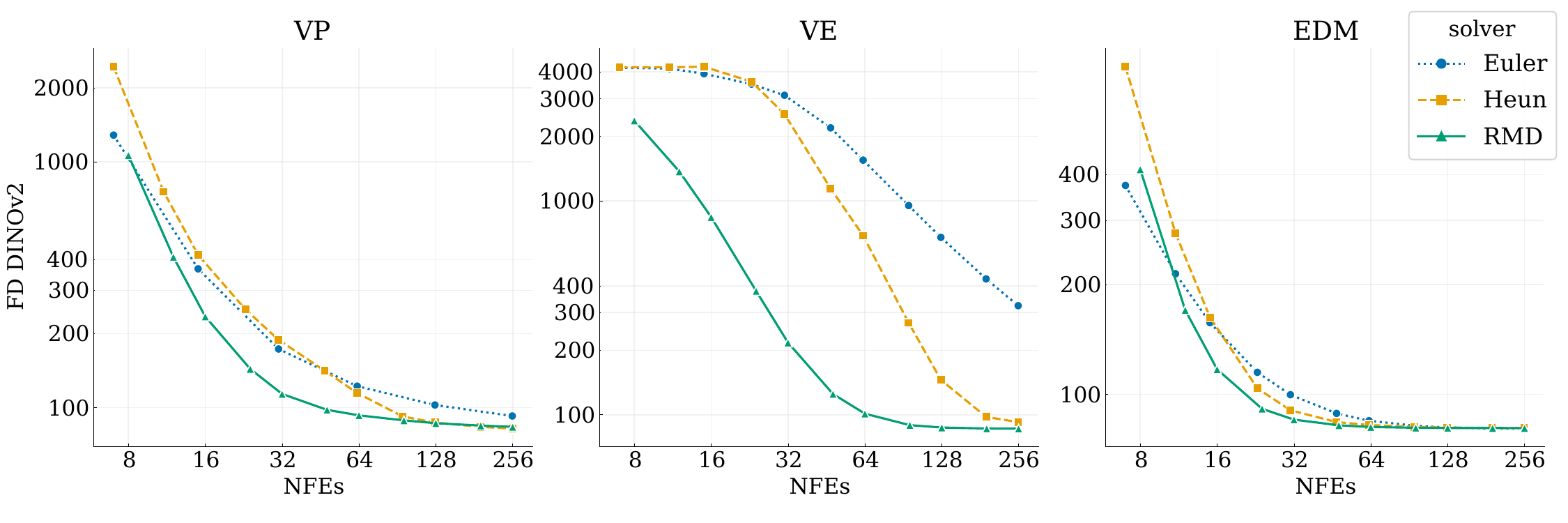}
  \caption{Image quality as measured by \( \fddino \). \ref{fig:ode_fid} uses the same generated images.}
  \label{fig:ode_fd_dinov2}
\end{figure}

\begin{figure}[ht]
  \centering
  \includegraphics[width=\textwidth]{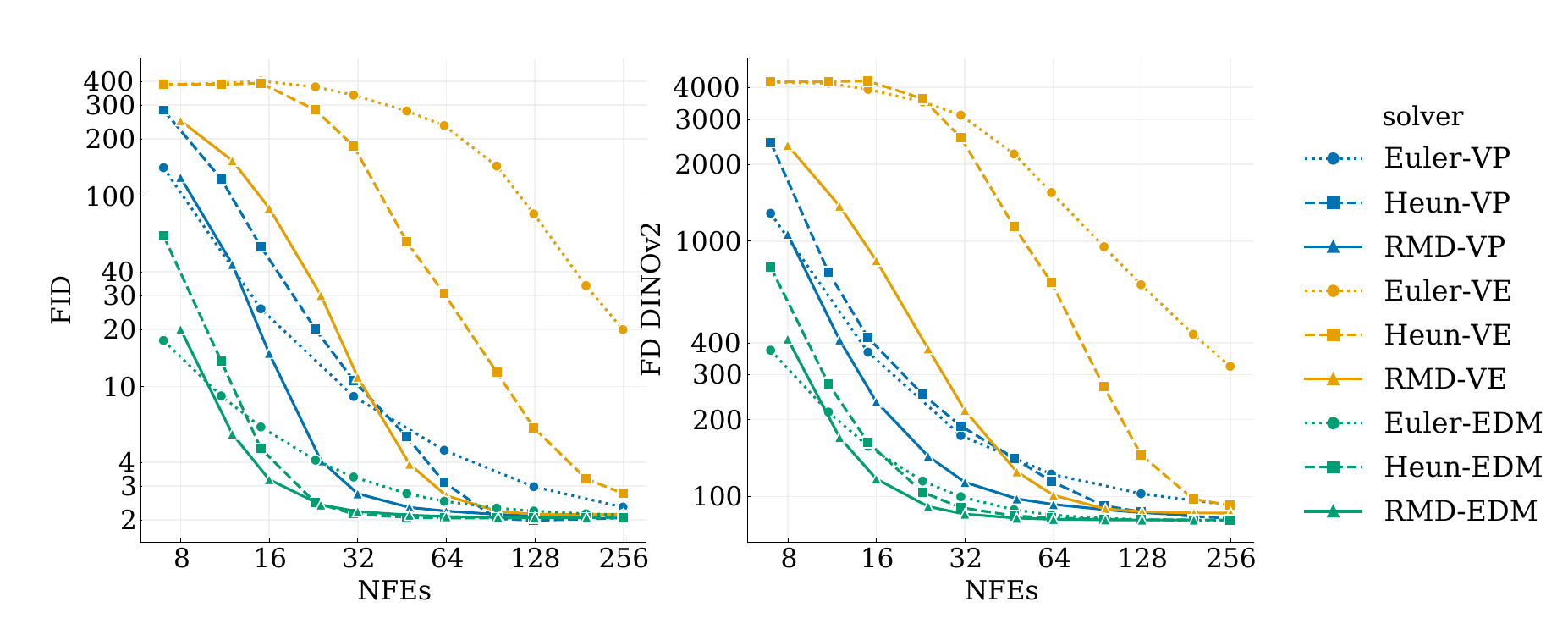}
  \caption{A variant of~\ref{fig:ode_fid} with all methods and settings shown on the same scale.}
  \label{fig:ode_fid_combine}
\end{figure}

\end{document}